\documentclass{article}
\usepackage[utf8]{inputenc}
\usepackage{amsmath}
\usepackage{amsfonts}
\usepackage{subfigure}
\usepackage{verbatim}
\usepackage{amssymb}
\usepackage{natbib}
\usepackage{fullpage}
\usepackage{color}

\usepackage{mathrsfs}
\usepackage{graphicx}
\usepackage{amsthm}
\usepackage{hyperref}
\usepackage{booktabs}
\usepackage{verbatim}
\usepackage{microtype}
\usepackage{bm}
\usepackage{algorithm,algorithmic}

\usepackage{hyperref}

\newtheorem{theorem}{Theorem}
\newtheorem*{theorem*}{Theorem}
\newtheorem{lemma}{Lemma}

\newtheorem*{lemma*}{Lemma}

\newcommand{\ten}[1]{\bm{\mathcal{#1}}}
\newcommand{\tuk}[4]{\ten{#1}(\bm{#2},\bm{#3},\bm{#4})}
\newcommand{\point}{\bm{\mathcal{S}}(\bm A, \bm B, \bm C)}
\newcommand{\tup}{\ten S, \bm A, \bm B, \bm C}
\newcommand{\diff}{\tuk{S}{A}{B}{C}-\ten T}
\newcommand{\so}{O^*}
\newcommand{\somega}{\Omega^*}
\newcommand{\stheta}{\Theta^*}
\newcommand{\reg}{l} 

\title{Optimization Landscape of Tucker Decomposition}
\author{Abraham Frandsen\thanks{Duke University} \and Rong Ge\thanks{Duke University}}

\begin{document}
\maketitle

\begin{abstract}
Tucker decomposition is a popular technique for many data analysis and machine learning applications. Finding a Tucker decomposition is a nonconvex optimization problem. As the scale of the problems increases, local search algorithms such as stochastic gradient descent have become popular in practice. In this paper, we characterize the optimization landscape of the Tucker decomposition problem. In particular, we show that if the tensor has an exact Tucker decomposition, for a standard nonconvex objective of Tucker decomposition, all local minima are also globally optimal. We also give a local search algorithm that can find an approximate local (and global) optimal solution in polynomial time.
\end{abstract}

\section{Introduction}

\newcommand{\R}{\mathbb{R}}

Tensor decompositions have been widely applied in data analysis and machine learning. In this paper we focus on Tucker decomposition~\citep{hitchcock1927expression,tucker1966some}. Tucker decomposition has been applied to TensorFaces~\citep{vasilescu2002multilinear}, data compression~\citep{wang2004compact}, handwritten digits~\citep{savas2007handwritten} and more recently to word embeddings~\citep{frandsen2019understanding}.

Unlike CP/PARAFAC~\citep{carroll1970analysis,harshman1970foundations} decomposition, Tucker decomposition can be computed efficiently if the original tensor has low rank. For example, this can be done by high-order SVD~\citep{de2000multilinear}. Many other algorithms have  also been proposed for tensor Tucker decomposition, see for example~\citep{de2000best,elden2009newton,phan2014fast}. 

In modern applications, the dimension of the tensor and the amount of data available are often quite large. In practice, simple local search algorithms such as stochastic gradient descent are often used. Even for matrix problems where exact solutions can be computed, local search algorithms are often applied directly to a nonconvex objective~\citep{koren2009bellkor,recht2013parallel}. Recently, a line of work~\citep{ge2015escaping,bhojanapalli2016global,sun2016complete,ge2016matrix,sun2016geometric,bandeira2016low} showed that although these problems have nonconvex objectives, they can still be solved by local search algorithms, because they have a simple {\em optimization landscape}. In particular, for matrix problems such as matrix sensing~\citep{bhojanapalli2016global,park2016non} and matrix completion~\citep{ge2016matrix,ge2017no}, 
it was shown that all local minima are globally optimal. Similar results were also known for special cases of tensor CP decomposition~\citep{ge2015escaping}.

In this paper, we prove similar results for Tucker decomposition. Given a tensor $\ten T \in \R^{d\times d\times d}$ with multilinear rank $(r,r,r)$, the Tucker decomposition of the tensor $\ten T$ has the form
$$
\ten T = \tuk{S^*}{A^*}{B^*}{C^*},
$$
where $\ten S^*\in \R^{r\times r\times r}$ is a core tensor, $\bm A^*, \bm B^*, \bm C^*\in \R^{r\times d}$ are three components (factor matrices). The notation $\tuk{S^*}{A^*}{B^*}{C^*}$ is a multilinear form defined in Section~\ref{sec:tensordef}. 

To find a Tucker decomposition by local search, the most straight-forward idea is to directly optimize the following nonconvex objective:
$$
L(\ten S,\bm A,\bm B,\bm C) = \|\ten T - \tuk{S}{A}{B}{C}\|_F^2.
$$

Clearly, $(\ten S^*,\bm A^*, \bm B^*, \bm C^*)$ is a global minimizer. However, since the optimization problem is nonconvex, it is unclear whether any local search algorithm can efficiently find a globally optimal solution. Our first result (Theorem~\ref{thm:exact}) shows that with an appropriate regularizer (designed in Section~\ref{sec:optimizationproblem}), all local minima of Tucker decomposition are globally optimal. 

The main difficulty of analyzing the optimization landscape of Tucker decomposition comes from the existence of {\em high order saddle points}. For example, when $\ten S, \bm A,\bm  B,\bm C$ are all equal to 0, any local movement of norm $\epsilon$ will only change the objective by at most $O(\epsilon^4)$. 
Characterizing the possible locations of such high order saddle points, and showing that they cannot become local minima is one of the major technical contributions of this paper.

In general, even if all local minima are globally optimal, a local search algorithm may still fail to find a global optimal solution due to high order saddle points. In the worst case it is known that 3rd order saddle points can be handled efficiently, while 4th order saddle point are hard to escape from~\citep{anandkumar2016efficient}. The objective $L$ has 4th order saddle points. However, our next result (Theorem~\ref{thm:robust}) shows that there is a specifically designed local search algorithm that can find an approximate global optimal solution in polynomial time. 

\section{Preliminaries}


\subsection{Tensor Notation and Basic Facts}
\label{sec:tensordef}
We use bold lower-case letters like $\bm{u}$ to denote vectors, bold upper-case letters like $\bm A$ to denote matrices, and bold caligraphic upper-case letters like $\ten T$ to denote tensors. We reserve the symbol $\bm I$ to denote the identity matrix; its particular dimension will be clear from context. 
Given a third order tensor $\ten S \in \mathbb{R}^{r_1\times r_2 \times r_3}$ and matrices $\bm A \in \mathbb{R}^{r_1\times d_1}$, $\bm B \in \mathbb{R}^{r_2\times d_2}$,
$\bm C \in \mathbb{R}^{r_3\times d_3}$, we define $\tuk{S}{A}{B}{C}  \in \mathbb{R}^{d_1\times d_2\times d_3}$ by 
\[
[\tuk{S}{A}{B}{C}]_{ijk} = \sum_{xyz} \ten S_{xyz}\bm A_{xi}\bm B_{yj}\bm C_{zk}.
\]
In the special case where one or more of $r_1, r_2, r_3$ equals $1$, we view $\tuk{S}{A}{B}{C}$ appropriately as a matrix, column vector, or scalar.
We equip $\mathbb{R}^{d_1\times d_2\times d_3}$ with the Frobenius inner product $\langle \cdot, \cdot \rangle$  and associated norm $\|\cdot \|_F$ given by
\begin{align*}
\langle \ten X, \ten Y \rangle &= \sum_{i,j,k = 1}^{d_1,d_2,d_3} \ten X_{ijk}\ten Y_{ijk} & \|\ten X\|_F = \sqrt{\langle \ten X, \ten X \rangle}
\end{align*}
We also define the operator $2$-norm $\|\cdot \|_2$ (i.e. the spectral norm) by
\[
\|\ten X\|_2 = \sup\left\{\ten X (\bm u,\bm v,\bm w) \,:\, \|\bm u\|_2 = \|\bm v\|_2 = \|\bm w\|_2 = 1\right\} 
\]
These two norms are related as follows \citep{wang2017operator}:$$
\left(\frac{\max(d_1,d_2,d_2)}{d_1d_2d_3}\right)^{1/2}\|\ten X\|_F \leq \|\ten X\|_2 \leq \|\ten X\|_F.
$$
In the special case of $d_1=d_2=d_3 = d$, we have $\|\ten X\|_F \leq d\|\ten X\|_2$.
Another important fact is that for $\sigma = \|\ten X\|_2$, there exist unit vectors $\bm u \in \mathbb{R}^{d_1}$, $\bm v \in \mathbb{R}^{d_2}$, and $\bm w \in \mathbb{R}^{d_3}$
such that the following hold \citep{lim2005singular}:
\begin{align*}
\tuk{X}{u}{v}{w} = \sigma\quad \tuk{X}{I}{v}{w} = \sigma \bm u\quad \tuk{X}{u}{I}{w}= \sigma \bm v\quad \tuk{X}{u}{v}{I} = \sigma \bm w
\end{align*}
Let $\ten X_{(i)} \in \mathbb{R}^{d_i\times \Pi_{j\neq i}d_j}$ denote the factor-$i$ flattening of $\ten X$ (for $i=1,2,3$).
We say $\ten X$ has multilinear rank $(r_1,r_2,r_3)$ if 
there exists a tensor $\ten S \in \mathbb{R}^{r_1\times r_2\times r_3}$  and matrices $\bm A \in \mathbb{R}^{r_1\times d_1}$, $\bm B \in \mathbb{R}^{r_2\times d_2}$, and $\bm C \in \mathbb{R}^{r_3\times d_3}$ of minimal dimension such that $\ten X = \tuk{S}{A}{B}{C}$. 
The tuple $(\ten S,\bm A,\bm B,\bm C)$ gives a \emph{Tucker decomposition} of $\ten X$.
Note that $\ten X_{(i)}$ has rank $r_i$ and $\tuk{X}{A}{B}{C}_{(1)} = \bm A^\top \ten X_{(1)} (\bm B\otimes \bm C)$ where $\otimes$ denotes the Kronecker product of matrices.

The space of parameters for our objective function is $\mathbb{R}^{r_1\times r_2\times r_3} \times \mathbb{R}^{r_1\times d_1}\times \mathbb{R}^{r_2\times d_2}\times \mathbb{R}^{r_3\times d_3}$. We write a point in this space as $(\tup)$, and equip it with inner product $\langle (\tup), (\ten S', \bm A', \bm B', \bm C')\rangle = \langle \ten S, \ten S'\rangle +\langle \bm A, \bm A'\rangle + \langle \bm B, \bm B'\rangle + \langle \bm C, \bm C'\rangle$ and associated norm $$\|(\tup)\|_F = \sqrt{\|\ten S\|_F^2+\|\bm A\|_F^2+\|\bm B\|_F^2+\|\bm C\|_F^2}.$$

\subsection{Optimization Problem}
\label{sec:optimizationproblem}

For simplicity, in this paper we assume $r_1=r_2=r_3 = r$, and $d_1=d_2=d_3 = d$. It is easy to generalize the result to the case with different $r_i$'s and $d_i$'s. 
Let $\ten T \in \mathbb{R}^{d\times d\times d}$ be a fixed third order tensor with multilinear rank $(r,r,r)$ for $r < d$. A simple objective for tensor decomposition can be defined as:
\begin{equation}\label{eq:loss}
L(\ten S,\bm A,\bm B,\bm C) = \|\tuk{S}{A}{B}{C} - \ten T\|_F^2.
\end{equation}

Suppose $\ten T = \ten S^* (\bm A^*, \bm B^*, \bm C^*)$, then Equation \eqref{eq:loss} has a global minimum at $(\ten S^*, \bm A^*, \bm B^*, \bm C^*)$ with the minimum possible $L$ value 0. In fact, due to symmetry, we know there are many more global minimizers of $L$: for any invertible matrices $\bm {Q_A}, \bm {Q_B}, \bm {Q_C} \in \R^{r\times r}$, let $\bm S = \tuk{S^*}{Q_A}{Q_B}{Q_C}$, and $\bm A = \bm {Q_A}^{-1} \bm A^*$, $\bm B =\bm {Q_B}^{-1} \bm B^*$ and $\bm C = \bm {Q_C}^{-1} \bm C^*$, then we also have $\ten T = \tuk{S}{A}{B}{C}$. Therefore, the loss $L$ has infinitely many global optimal solutions.

The existence of many equivalent global optimal solutions causes problems for local search algorithms, especially simpler ones like gradient descent. The reason is that if we scale $\bm A, \bm B, \bm C$ with a large constant $c$, and scale $\ten S$ with $1/c^3$, the tensor $\tuk{S}{A}{B}{C}$ does not change. However, after this scaling the partial gradient of $\ten S$ is multiplied by $c^3$, while the partial gradients of $\bm A,\bm B, \bm C$ are multiplied by $1/c$. When $c$ is large one has to choose a very small step size for gradient descent, and this results in very slow convergence. 

We address the problem of scaling by introducing a regularizer $\reg(\tup)$ given by 
\begin{equation} \label{eq:reg}
\|\bm A\bm A^\top - \ten S_{(1)}\ten S_{(1)}^\top\|_F^2 + \|\bm B\bm B^\top - \ten S_{(2)}\ten S_{(2)}^\top\|_F^2 + \|\bm C\bm C^\top - \ten S_{(3)}\ten S_{(3)}^\top\|_F^2.
\end{equation}

Intuitively, the three terms in the regularizer ensure that $\bm A$ and $\ten S$ (similarly, $\bm B, \bm C$ and $\ten S$) have similar norms. Similar regularizers were used for analyzing the optimization landscape of asymmetric matrix problems\citep{park2016non}, where the same scaling problem exists. However, to the best of our knowledge we have not seen this regularizer used for Tucker decomposition.

For technical reasons that will become clear in Section~\ref{sec:landscape} (especially in Lemma~\ref{lem:reg_perturb}), we actually use $R(\ten S,\bm A,\bm B,\bm C) = \reg(\ten S,\bm A,\bm B,\bm C)^2$ as the regularizer with weight $\lambda > 0$, so the final optimization problem we consider is:

%
\begin{equation} \label{eqn:obj}
\underset{\ten S, \bm A, \bm B, \bm C}{\min} L(\ten S,\bm A,\bm B,\bm C) + \lambda R(\ten S,\bm A,\bm B,\bm C).
\end{equation}

Note that even for Equation~\eqref{eqn:obj}, there are still infinitely many global minimizers. In particular, one can rotate $\bm A$ and $\ten S$ (similarly, $\bm B, \bm C$ and $\ten S$) simultaneously to get equivalent solutions. 
A priori it is unclear whether there always exists a global minimizer that achieves 0 loss for Equation~\eqref{eqn:obj}. Our proof in Section~\ref{sec:landscape} implicitly shows that such a solution must exist.

\section{Characterization of Optimization Landscape}
\label{sec:landscape}

In this section, we analyze the optimization landscape for the objective \eqref{eqn:obj} for Tucker decomposition. In particular, we establish the following result.
\begin{theorem} \label{thm:exact}
For any fixed $\lambda > 0$, all local minima of the objective function $f= L + \lambda R$ as in Equation \eqref{eqn:obj} have loss 0.
\end{theorem}

Note that the theorem would not hold for $\lambda = 0$ (when there is no regularizer). A counter-example is when $\ten T = \bm a^*\otimes \bm b^*\otimes \bm c^*$ for some unit vectors $\bm a^*, \bm b^*, \bm c^*$, and $\ten S = \bm 0, \bm A = \bm a^\top, \bm B = \bm b^\top, \bm C = \bm c^\top$ where $\bm a,\bm b,\bm c$ are unit vectors that are orthogonal to $\bm a^*, \bm b^*, \bm c^*$ respectively. A local change will have no effect if the new $\bm S$ is still $\bm 0$, and will make the objective function larger if the new $\bm S$ is nonzero.

In order to prove this theorem, we demonstrate a direction of improvement for all points $(\tup)$ that don't achieve the global optimum.
A direction of improvement is a tuple $(\Delta\ten S, \Delta\bm A, \Delta\bm B, \Delta\bm C)$ such that 
\[
f(\ten S+\epsilon \Delta\ten S,\bm  A + \epsilon \Delta\bm  A,\bm  B + \epsilon \Delta\bm  B, \bm C + \epsilon \Delta\bm  C) < f(\tup)
\]
for all sufficiently small $\epsilon > 0$. Clearly, if a point $(\tup)$ has a direction of improvement, then it cannot be a local minimum.

Throughout the section, 
let $\bm P_1$ ($\bm P_2, \bm P_3$) be the projection onto the column span of $\ten T_{(1)}$ ($\ten T_{(2)}, \ten T_{(3)}$).   
Let $\bm A_1 = \bm A\bm P_1$ and $\bm A_2 = \bm A(\bm I-\bm P_1)$ (similarly for $\bm B,\bm C$). The proof works in the following 4 steps:

\vspace*{0.1in}
{\noindent \bf Bounding the regularizer} First we show that when $\nabla f = \bm 0$, the regularizer $R$ must be equal to 0 (Lemm~\ref{lem:nonzeroreg} in Section~\ref{sec:nonzeroreg}). At a high level, this is because the gradient of regularizer $R$ is always orthogonal to the gradient of main term $L$. Therefore if the gradient of the entire objective is $\bm 0$, the gradient of $R$ must also be $\bm 0$. We complete the proof by showing that $\nabla R = \bm 0$ implies $R = 0$.

\vspace*{0.1in}
{\noindent \bf Removing extraneous directions}
Next, we show that when $\nabla f = \bm 0$, the projection in the wrong subspaces $\bm A_2,\bm B_2, \bm C_2$ are all equal to $\bm 0$. 
This is because the direction of directly removing the projection in the wrong subspace $\bm A_2$ is a direction of improvement (see Lemma~\ref{lem:a21}). 


\vspace*{0.1in}
{\noindent \bf Adding missing directions} After the previous steps, we know that the rows of $\bm A$ are in the column span of $\ten T_{(1)}$. However, the row span of $\bm A$ might be smaller. In this case, there exist directions $\bm a, \bm b, \bm c$ such that $\ten T(\bm a,\bm b,\bm c) > 0$, and $\bm A\bm a = \bm 0$. We will show that in this case we can always add the missing directions into $\bm A$ and $\ten S$. 
This is the most technical part of our proof, and high order stationary points may appear when $\bm B\bm b$ or $\bm C\bm c$ are also $\bm 0$. See Sections~\ref{sec:missingdirections}.

\vspace*{0.1in}
{\noindent \bf Fixing $\ten S$} Finally, we know that the components $\bm A,\bm B, \bm C$ must span the correct subspaces. Our final step shows that in this case, if $L > 0$ then it is easy to find a direction of improvement, see Section~\ref{sec:handlings}.

\subsection{Direction of Improvement for Points with Nonzero Regularizer}\label{sec:nonzeroreg}
We show any point with nonzero regularizer must also have a nonzero gradient, therefore the (negative) gradient itself is a direction of improvement. 

\begin{lemma} \label{lem:nonzeroreg}
For any $\tup$, if $R(\tup) > 0$ then $\|\nabla f\| > 0$.
\end{lemma}

To prove this, we first show that if the regularizer is nonzero, then its gradient is nonzero. 
\begin{lemma}
\label{lem:reg_gradient}
The function $\reg$ satisfies
\[
4\reg(\tup) = \langle \nabla_{\bm A} \reg, \bm A\rangle + \langle \nabla_{\bm B} \reg, \bm B\rangle + \langle \nabla_{\bm C} \reg, \bm C\rangle + \langle \nabla_{\ten S} \reg, \ten S\rangle 
\]
\end{lemma}
\begin{proof}
Note the following calculations:
\begin{align*}
\langle \nabla_{\bm A} \reg, \bm A\rangle &=  \langle   4(\bm A \bm A^\top- \ten S_{(1)}\ten S_{(1)}^\top)\bm A,\bm A\rangle\\
&= 4\langle \bm A \bm A^\top - \ten S_{(1)}\ten S_{(1)}^\top,\bm A\bm A^\top\rangle\\
\langle 4\ten S(\ten S_{(1)}\ten S_{(1)}^\top - \bm A\bm A^\top,\bm I,\bm I), \ten S\rangle &= -4\langle \bm A\bm A^\top - \ten S_{(1)}\ten S_{(1)}^\top,\ten S_{(1)}\ten S_{(1)}^\top\rangle
\end{align*}
The left-hand side above is one of the terms in $\nabla_{\ten S}\reg$. Doing the same calculation for the other modes and then adding everything together yields the result.
\end{proof}

We next show that the gradient of the regularizer is always orthogonal to the gradient of the main term (i.e. the tensor loss $L$).
\begin{lemma}
\label{lem:reg_orthogonality}
For any $\tup$, $\langle \nabla L(\tup), \nabla R(\tup)\rangle = 0$.
\end{lemma}
\begin{proof}
We start by calculating the partial gradients for $L$ and $r$. We have
\begin{align*}
&\begin{aligned}[c]
\nabla_{\bm A} L &= 2\ten S_{(1)}(\bm B\otimes \bm C)(\diff)_{(1)}^\top\\
\nabla_{\bm B} L &= 2\ten S_{(2)}(\bm A\otimes \bm C)(\diff)_{(2)}^\top\\
\nabla_{\bm C} L &= 2\ten S_{(3)}(\bm A\otimes \bm B)(\diff)_{(3)}^\top\\
\end{aligned}\quad
\begin{aligned}[c]
\nabla_{\bm A} \reg &= 4(\bm A\bm  A^\top- \ten S_{(1)}\ten S_{(1)}^\top)\bm A\\
\nabla_{\bm B} \reg &= 4(\bm B \bm B^\top - \ten S_{(2)}\ten S_{(2)}^\top)\bm B\\
\nabla_{\bm C} \reg &= 4(\bm C\bm C^\top - \ten S_{(3)}\ten S_{(3)}^\top)\bm C\\
\end{aligned}\\
&\nabla_{\ten S} L = 2(\tuk{S}{A}{B}{C}-\ten T)(\bm A^\top,\bm B^\top,\bm C^\top)\\
&\nabla_{\ten S} \reg = 4\ten S(\ten S_{(1)}\ten S_{(1)}^\top - \bm A\bm A^\top,\bm I,\bm I) + 4\ten S(\bm I,\ten S_{(2)}\ten S_{(2)}^\top-\bm B\bm B^\top,\bm I)\\
&\qquad\quad+4\ten S(\bm I,\bm I,\ten S_{(3)}\ten S_{(3)}^\top-\bm C\bm C^\top)
\end{align*}
We now compute the following:
\begin{align*}
\langle \nabla_{\bm A} L, \nabla_{\bm A} \reg \rangle &=
 8\langle  \ten S_{(1)}(\bm B\otimes \bm C)(\diff)_{(1)}^\top, (\bm A\bm  A^\top- \ten S_{(1)}\ten S_{(1)}^\top)\bm A 	 \rangle\\
&= 8\langle (\diff)(\bm A^\top,\bm B^\top,\bm C^\top), \ten S(\bm A\bm A^\top - \ten S_{(1)}\ten S_{(1)}^\top,\bm I,\bm I)\rangle
\end{align*}
From here it is easy to see that 
$\langle \nabla_{\ten S} L, \nabla_{\ten S} \reg\rangle = - \langle \nabla_{\bm A} L, \nabla_{\bm A} \reg\rangle - \langle \nabla_{\bm B} L, \nabla_{\bm B} \reg\rangle - \langle \nabla_{\bm C} L, \nabla_{\bm C} \reg\rangle$, therefore $\langle \nabla L, \nabla \reg\rangle = 0$. 
Since $\nabla R = 2\reg\nabla \reg$, the result follows.
\end{proof}

Now we are ready to prove Lemma~\ref{lem:nonzeroreg}:
\begin{proof}
By Lemma \ref{lem:reg_orthogonality}, we know 
$\|\nabla f\|_F^2 = \|\nabla L\|_F^2 + \|\nabla R\|_F^2$.
On the other hand, by Lemma \ref{lem:reg_gradient} and an application of the Cauchy-Schwarz inequality, we see that
\[
\|\nabla \reg\|_F\|(\tup)\|_F \geq 4\reg(\tup),
\]
which means that $\|\nabla \reg\|_F > 0$ whenever $R = \reg^2 > 0$. 
But $\nabla R = 2\reg\nabla \reg$, so we have that $\|\nabla R\|_F > 0$, whence $\nabla f \neq \bm 0$.
\end{proof}

To facilitate later proofs, we will also show a fact that if one perturbs a solution with 0 regularizer, then the regularizer remains very small.

\begin{lemma}\label{lem:reg_perturb}
If $R = 0$, and $\|\Delta \bm A\|_F+\|\Delta \bm B\|_F + \|\Delta \bm C\|_F+\|\Delta \ten S\|_F \le O(1)$, then $R(\ten S+\epsilon\Delta \ten S, \bm A+\epsilon \Delta \bm A, \bm B+\epsilon \Delta \bm B, \bm C+\epsilon \Delta \bm C) = O(\epsilon^4)$ for sufficiently small $\epsilon$.
\end{lemma}

\begin{proof}
It suffices to check that the term $\|(\bm A+\epsilon \Delta \bm A)(\bm A+\epsilon \Delta \bm A)^\top - (\ten S+\epsilon\Delta \ten S)_{(1)}(\ten S+\epsilon\Delta \ten S)_{(1)}^\top\|_F = O(\epsilon)$, as other terms are symmetric, and the final $R$ is degree 4 over these terms. This is clear as 
we know $\|\bm A\bm A^\top - \ten S_{(1)}\ten S_{(1)}^\top\|_F = 0$ because $R=0$, and all the remaining terms are bounded by $O(\epsilon)$.
\end{proof}

\subsection{Removing Extraneous Directions}\label{sec:extradirections}
In this section, we show that if  $\bm A$ (respectively $\bm B$, $\bm C$) has a direction in its row-space that is perpendicular to the column-space of $\ten T_{(1)}$ (respectively $\ten T_{(2)}$, $\ten T_{(3)}$), then we have a direction of improvement. 
In particular, our goal is to show $\bm A_2 = \bm 0$ for all local minima (symmetric arguments will then show $\bm B_2 = \bm C_2 = \bm 0$). We first show that $\ten S(\bm A_2,\bm B,\bm C) = \bm 0$.

\begin{lemma} \label{lem:a21}
Assume that $R(\tup) = 0$. If $\ten S(\bm A_2,\bm B,\bm C) \neq \bm 0$, then $\Delta \bm A = -\bm A_2$ is a direction of improvement.
\end{lemma}
\begin{proof}
Set $\Delta \bm A =  - \bm A_2$. 
Then for $\epsilon > 0$
\begin{align*}
L(\ten S,\bm A+\epsilon \Delta \bm A,\bm B,\bm C) &= \|\diff + \epsilon \ten S(\Delta \bm A,\bm B,\bm C)\|_F^2\\
&= L(\tup) -2\epsilon\|\ten S(\bm A_2,\bm B,\bm C)\|_F^2 + O(\epsilon^2),
\end{align*}
since $\langle \diff, \ten S(\bm A_2,\bm B,\bm C)\rangle = \langle \ten S(\bm A_2,\bm B,\bm C), \ten S(\bm A_2,\bm B,\bm C)\rangle$.
Hence, for all sufficiently small $\epsilon$, $L(\ten S,\bm A+\epsilon \Delta \bm A,\bm B,\bm C) < L(\tup)$.
By Lemma~\ref{lem:reg_perturb} we know $R(\ten S,\bm A+\epsilon \Delta \bm A, \bm B, \bm C) = O(\epsilon^4)$.
Hence, for sufficiently small $\epsilon$, the decrease in $L$ will exceed any increase in $R$.
This shows that $\Delta \bm A$ is a direction of improvement.
\end{proof}

We next establish that $R(\tup) =0$ and $\ten S(\bm A_2,\bm B,\bm C) = \bm 0$ together imply that $\bm A_2 = \bm 0$. 

%


\begin{lemma} \label{lem:zeroa2}
If $R(\tup) = 0$ and $\ten S(\bm A_2,\bm B,\bm C) = \bm 0$, then $\bm A_2 = \bm 0$.
\end{lemma}
\begin{proof}
Since $R(\tup) = 0$, we have $\bm B\bm B^\top = \ten S_{(2)} \ten S_{(2)}^\top$ and $\bm C\bm C^\top = \ten S_{(3)} \ten S_{(3)}^\top$. This means the column span of $\ten S_{(2)}$ ($\ten S_{(3)}$) is the same as column span of $\bm B$ ($\bm C$).
Let $\bm B^+$ and $\bm C^+$ denote the pseudoinverses of $\bm B$ and $\bm C$.
Note that the orthogonal projections onto the column-space of $\bm B$ and $\bm C$ are given by
$\bm P_{\bm B} := \bm B\bm B^+$ and $\bm P_{\bm C} := \bm C\bm C^+$, respectively.
Using these facts along with $\ten S(\bm A_2,\bm B,\bm C) =\bm 0$, we have 
\begin{align*}
\bm 0 &= \ten S(\bm A_2,\bm B,\bm C)(\bm I,\bm B^+,\bm C^+) = \ten S(\bm A_2,\bm P_{\bm B},\bm P_{\bm C}) = \ten S(\bm A_2,\bm I,\bm I).
\end{align*}
Using the fact that $\ten S_{(1)}\ten S_{(1)}^\top = \bm A \bm A^\top = \bm A_1 \bm A_1^\top  + \bm A_2 \bm A_2^\top$, we have
\[
\|\bm A_2\bm A_2^\top\|_F^2 \le \langle\bm A_2\bm A_2^\top, \ten S_{(1)}\ten S_{(1)}^\top\rangle = \|\ten S(\bm A_2, \bm I,\bm I)\|_F^2 = 0,
\]
which, in particular, means that $\bm A_2 = \bm 0$.
\end{proof}


\subsection{Adding Missing Directions}
\label{sec:missingdirections}

We now consider the case where the row-spans of $\bm A$, $\bm B$, and $\bm C$ are not equal to the column-spans of $\ten T_{(1)}$,$ \ten T_{(2)}$, and $\ten T_{(3)}$, respectively. Again by symmetry, we focus on the case when row-span of $\bm A$ is not equal to column-span of $\ten T_{(1)}$. 

\begin{lemma}\label{lem:missingdirections} If the row-span of $\bm A$ is a strict subset of the column-span of $\ten T_{(1)}$ and $R = 0$, then there is a direction of improvement. 
\end{lemma}

\begin{proof}
If the row-span of $\bm A$ is a strict subset of column-span of $\ten T_{(1)}$, we must have a vector $\bm a$ that is in the column-span of $\ten T_{(1)}$, but $\bm A \bm a = \bm 0$. For this vector we know $\ten T(\bm a, \bm I, \bm I) \ne \bm 0$, therefore there must exist vectors $\bm b,\bm c$ such that $\ten T(\bm a,\bm b, \bm c) > 0$. 
This is true even if we restrict $\bm b$ to be either in the row span of $\bm B$ or to satisfy $\bm B b = 0$ (and similarly for $\bm c$), as we can partition the matrix into 4 subspaces based on the projections of its columns to row span of $\bm B$ (and its rows to row span of $\bm C$).
In particular, if we let $\bm b_1$ be the projection of $\bm b$ onto the row-span of $\bm B$ and $\bm c_1$ be the projection of $\bm c$ onto the row-span of $\bm C$, and set $\bm b_2 = \bm b - \bm b_1$ and $\bm c_2 = \bm c - \bm c_1$, then we have $\ten T(\bm a, \bm b, \bm c) = \sum_{i,j\in \{1,2\}} \ten T(\bm a, \bm b_i, \bm c_j).$
Hence, $\ten T(\bm a, \bm b_i, \bm c_j) > 0$ for some choice of $i, j \in \{1,2\}$.

\vspace*{0.1in}
{\noindent \bf One missing direction}
In this case  $\bm b$ and $\bm  c$ are in row span of $\bm B,\bm C$ respectively. 
Choose unit vectors $\bm u,\bm v,\bm w\in \mathbb{R}^r$ such that $\bm A^\top \bm u = \bm 0$, $\bm B^\top \bm v = \alpha_1 \bm b$, and $\bm C ^\top \bm w = \alpha_2\bm  c$, where $\alpha_1$ and $\alpha_2$ are positive real numbers. 
Consider the directions $\Delta \bm A =\bm u\bm a^\top$, $\Delta \ten S =  \bm u\otimes \bm v\otimes \bm w$.
Observe that $\Delta \point = \bm A^\top\bm u\otimes \bm B^\top\bm v\otimes \bm C^\top\bm w = \bm 0$ and $\ten S(\Delta \bm A,\bm  B, \bm C) = \bm 0$ since the column-space of $\bm S_{(1)}$ is equal to the column-space of $\bm A$. 
Moreover, $\Delta \ten S(\Delta \bm A, \bm B, \bm C) = \bm a\otimes \bm B^\top \bm v\otimes \bm C^\top \bm w = \alpha_1\alpha_2\bm a\otimes \bm b \otimes \bm c$.
Hence,  for $\epsilon > 0$, we have
\begin{align*}
L(\ten S + \epsilon \Delta \ten S, \bm A + \epsilon \Delta \bm A, \bm B, \bm C) &= \|\diff + \epsilon^2\Delta\ten S(\Delta \bm A, \bm B, \bm C)\|_F^2\\
&= L(\tup)-2\epsilon^2\alpha_1\alpha_2\ten T(\bm a,\bm b,\bm c) + O(\epsilon^4).
\end{align*}
On the other hand, by Lemma~\ref{lem:reg_perturb} $R(\ten S + \epsilon\Delta\ten S, \bm A + \epsilon \Delta \bm A, \bm B, \bm C) = O(\epsilon^4)$ since $R(\tup) = 0$.
Hence, for small enough $\epsilon$, the improvement in the tensor loss dominates all other perturbations, so we have a direction of improvement. 

\vspace*{0.1in}
{\noindent \bf Two missing directions}
Now assume that $\bm A\bm a = \bm B\bm b = \bm 0$, and $\bm c$ is in the row span of $\bm C$. 
Choose unit vectors $\bm u,\bm v,\bm w \in \mathbb{R}^r$ such that $\bm A^\top \bm u = \bm B^\top \bm v =\bm 0$ and $\bm C^\top \bm w = \alpha \bm c$ where $\alpha > 0$.
Consider the directions $\Delta \bm A =   \bm u\bm a^\top$, $\Delta \bm B = \bm v\bm b^\top$, $\Delta \ten S = \bm u\otimes \bm v\otimes \bm w$.
Through a very similar calculation as in the previous case,
\[
L(\ten S+\epsilon \Delta \ten S, \bm A+\epsilon \Delta \bm A, \bm B+\epsilon \Delta \bm B,\bm C) = L(\tup) - 2\epsilon^3\alpha \ten T(\bm a,\bm b,\bm c) + \epsilon^6\alpha^2.
\]
As before, by Lemma~\ref{lem:reg_perturb} $R(\ten S + \epsilon\Delta\ten S, \bm A + \epsilon \Delta \bm A, \bm B +\epsilon \Delta \bm B, \bm C) = O(\epsilon^4)$.
Hence, the decrease in the tensor loss dominates all other perturbations for sufficiently small $\epsilon$, and so this is a direction of improvement. Note that in this case the amount of improvement is $\Theta(\epsilon^3)$, so the point is a 3rd order saddle point.

The case where $\bm C \bm c = \bm 0$ and $\bm b$ is in the row-span of $\bm B$ is similar, and likewise yields a direction of improvement.

\vspace*{0.1in}
\noindent
{\bf Three missing directions}
Now assume that $\bm A\bm a = \bm B\bm b = \bm C\bm c = \bm 0$, and choose unit vectors $\bm u,\bm v,\bm w \in \mathbb{R}^r$ such that $\bm A^\top \bm a  = \bm B^\top \bm v = \bm C^\top \bm w = \bm 0$.
Consider the  directions $\Delta \bm A = \bm u\bm a^\top$, $\Delta \bm B = \bm v\bm b^\top$, $\Delta \bm C = \bm w\bm c^\top$, and $\Delta \ten S = \bm u\otimes\bm  v\otimes \bm w$.
Once again, most perturbations in the tensor loss vanish, and we have
\[
L(\ten S+\epsilon \Delta \ten S, \bm A+\epsilon\Delta \bm A, \bm B+\epsilon \Delta \bm B,\bm C+\epsilon \Delta \bm C) = L(\tup) - 2\epsilon^4 \ten T(\bm a,\bm b,\bm c) + \epsilon^8.
\]
In this case, the regularizer doesn't change at all,
since $\Delta \ten S_{(i)}\ten S_{(i)}^\top = \bm 0$ for $i = 1,2,3$, $\Delta \bm A \bm A^\top = \Delta \bm B \bm B^\top = \Delta \bm C \bm C^\top =\bm  0$, 
and $\Delta \bm A\Delta \bm A^\top - \Delta \ten S_{(1)}\Delta \ten S_{(1)}^\top  = \bm u\bm u^\top - \bm u \bm u^\top =\bm 0$ (and the two other analogous terms likewise vanish).
Hence, for sufficiently small $\epsilon$, the objective function decreases, so this is a direction of improvement. This point is a 4th order saddle point.

\end{proof}

\subsection{Improving the core tensor}
\label{sec:handlings}
We finally consider the case where the matrices $\bm A, \bm B, \bm C$ have the correct row-spaces but $\point \neq\ten T$. In this situation, we can make 
progress by changing only $\ten S$.

\begin{lemma}\label{lem:handlings}
If $R = 0$, row spans of $\bm A,\bm B,\bm C$ are equal to column span of $\ten T_{(1)},\ten T_{(2)}, \ten T_{(3)}$ respectively, but $L > 0$, then there exists a direction of improvement.
\end{lemma}


\begin{proof}
Since the spans of $\bm A,\bm B, \bm C$ are already correct, let $\bm A^+$ be the pseudoinverse of $\bm A$, then if we let $\ten S' = \ten T(\bm A^+, \bm B^+, \bm C^+)$, we have $\ten S'(\bm A,\bm B,\bm C) = \ten T$.
Consider the direction $\Delta \ten S = \ten S' - \ten S$.
\begin{align*}
L(\ten S+\epsilon \Delta \ten S, \bm A, \bm B, \bm C) &= \|(1-\epsilon)\point - (1-\epsilon)\ten T\|_F^2\\
&= (1-\epsilon)^2L(\tup).
\end{align*}

For the regularizer $R$, again by Lemma~\ref{lem:reg_perturb} we have $R(\ten S+\epsilon \Delta \ten S, \bm A, \bm B, \bm C) = O(\epsilon^4)$. Hence, this is a direction of improvement.
\end{proof}

\subsection{Proof of Main Theorem}

Now with all the lemmas we are ready to prove the main theorem:

\begin{proof}[Proof of Theorem~\ref{thm:exact}] The Theorem follows immediately from the sequence of lemmas.

First, by Lemma~\ref{lem:nonzeroreg}, we know any local minima must satisfy $R = 0$. Next, by Lemma~\ref{lem:zeroa2} and Lemma~\ref{lem:a21}, we know the row spans of $\bm A$, $\bm B$, $\bm C$ must be subsets of column spans of $\ten T_{(1)}, \ten T_{(2)}, \ten T_{(3)}$ respectively. In the third step, by Lemma~\ref{lem:missingdirections}, we further show that  the row spans of $\bm A$, $\bm B$, $\bm C$ must be exactly equal to column spans of $\ten T_{(1)}, \ten T_{(2)}, \ten T_{(3)}$ respectively. Finally, by Lemma~\ref{lem:handlings} we know the loss function must be equal to 0.
\end{proof}

\section{Escaping from High Order Saddle Points for Tucker Decomposition}
\label{sec:inexact}

As we discussed before, since our objective $f = L+\lambda R$ as in \eqref{eqn:obj} may have high order saddle points, standard local search algorithms may not be able to find a local minimum. However, in this section we show that the high order saddle points of $f$ are {\em benign}: there is a polynomial time local search algorithm that can find an approximate local and global minimum of $f$.

We will first review the guarantees of standard local search algorithms, and then describe how to escape from high order saddle points.

\subsection{Local search algorithms for second order stationary points}

For a general function $f(\bm x)$ whose first two derivatives exist, we say a point $\bm x$ is a $(\tau_1,\tau_2)$-second order stationary point if
$$
\|\nabla f(\bm x)\| \le \tau_1, \quad \lambda_{min}(\nabla^2 f(\bm x)) \ge -\tau_2.
$$
If the function $f(\bm x)$ satisfies the gradient and Hessian Lipschitz conditions
\begin{align*}
\forall \bm x, \bm y & \quad \|\nabla f(\bm x) - \nabla f(\bm y)\| \le \rho_1 \|\bm x-\bm y\|_2,\\
\forall \bm x, \bm y & \quad \|\nabla^2 f(\bm x) - \nabla^2 f(\bm y )\| \le \rho_2 \|\bm x-\bm y\|_2,
\end{align*}
there are many local search algorithms that can find $(\tau_1,\tau_2)$-second order stationary points in polynomial time. This includes traditional second order algorithms such as cubic regularization~\citep{nesterov2006cubic}, and more recently first order algorithms such as perturbed gradient descent~\citep{jin2017escape}.

Of course, these guarantees are not enough for our objective $f$, as it has higher order saddle points. The main theorem in this section shows that there is an efficient local search algorithm that can optimize $f$.
\begin{theorem}\label{thm:robust}
Let $\lambda = 1/16r^4$, assume wlog. that $\|\ten T\|_F = 1$ and the initial point satisfies $f = L+\lambda R= O(1)$\footnote{This can be achieved by initializing at 0, or any point with norm $O(1)$.}. Then there is a local search algorithm that in $\mbox{poly}(d, r, 1/\epsilon)$ time finds a point $(\tup)$ such that $f(\tup) \le \epsilon$.
\end{theorem}

The algorithm that we will design is just a proof of concept: although its running time is polynomial, it is far from practical. We have not attempted to improve the dependencies on $d, r, 1/\epsilon$. Local search algorithms seem to perform much better for Tucker decomposition in practice, and understanding that is an interesting open problem.

To prove Theorem~\ref{thm:robust}, we will first show that sublevel sets of $f$ are all bounded (Section~\ref{subsec:bounded}). 
This allows us to bound the gradient and Hessian Lipschitz constants $\rho_1$ and $\rho_2$, so we can use any of the previous local search algorithm to find a $(\tau_1,\tau_2)$-second order stationary point.

Next, we follow the steps of Theorem~\ref{thm:exact}, but we do it much more carefully to show that as long as the objective is larger than $\epsilon$, then either the point has a large gradient or a negative eigenvalue in Hessian, or there is a way to construct a direction of improvement. This is captured in our main Lemma~\ref{lem:robustmain}.

Finally we give a sketch of the algorithm and show that these local improvements are enough to guarantee the convergence in Section~\ref{sec:algdesc}.

Throughout the section, we use $\so(\cdot)$, $\somega(\cdot)$ and $\stheta(\cdot)$ to hide polynomial factors of $r$ and $d$. 
We will introduce several numerical quantities in the remainder of this section; we list the most important ones in Table \ref{tab:parameters}.

\begin{table}
\centering
\begin{tabular}{ccl}
Symbol & Definition & Note\\
\hline
$\lambda$ & $\frac{1}{16r^4}$ & weight for regularizer\\
$K$ & $O^*(1)$ & universal bound for norms of $\tup,\ten T$\\
$\tau$ & $< 1$ & bound on $R (\tup)$\\
$\gamma$ & $\stheta(\tau^{1/48})$ & bound on the norm  of $\bm A_3, \bm B_3, \bm C_3$, introduced in Lemma \ref{lem:removing_directions}\\
$\sigma$ & $\sqrt{\gamma}$ & singular value threshold for $\bm A_1, \bm B_1, \bm C_1$\\
$\kappa_0$ & $\sqrt{\gamma}$ & max error in $\ten T_{1,1,1}$, introduced in Lemma \ref{lem:improve_S}\\
$\kappa_1$ & $2K\sigma^{3/4}$ & max error in $\ten T_{2,1,1}$, introduced in Lemma \ref{lem:add_1_direction}\\
$\kappa_2$ & $2K\sigma^{1/8}$ & max error in $\ten T_{2,2,1}$, introduced in Lemma \ref{lem:add_2_directions}\\
$\kappa_3$ & $2K\sigma^{1/2}$ & max error in $\ten T_{2,2,2}$, introduced in Lemma \ref{lem:add_3_directions}
\end{tabular}
\caption{Notation and definitions used in Section \ref{sec:inexact}\label{tab:parameters}}
\end{table}

%
%

\subsection{Bounded Sublevel Sets}
\label{subsec:bounded}
We first establish the boundedness of sublevel sets of the objective function. 
Our local search algorithm will guarantee the function value decreases in every iteration, so the trajectory of the algorithm will remain in a sublevel set. As a result, we know that the parameters remain bounded in norm at each step by some constant, say $K$. 

\begin{lemma}
\label{lem:bounded_sublevel}
For all $\Gamma \geq 0$, the set of points $(\tup)$ with $f(\tup) \le \Gamma$ satisfy that $\|\ten S\|_F, \|\bm A\|_F, \|\bm B\|_F, \|\bm C\|_F \le K$ where $K = \so((\Gamma+1)^{1/8})$.
\end{lemma}

To prove this lemma, we will first state some tools that we need.

\begin{lemma}
\label{lemma:submult}
For any parameter tuple $(\tup)$, we have
\[
\|\point\|_F \leq \|\ten S\|_F\|\bm A\|_2\|\bm B\|_2\|\bm C\|_2.
\]
\end{lemma}
\begin{proof}
This follows from the fact that $\|\cdot\|_F$ is invariant to matricization, and the fact that $\|\bm P \bm Q\|_F \le \|\bm P\|_2 \|\bm Q\|_F$.
Observe that 
\[
\|\point\|_F = \|\bm A^\top \ten S(\bm I,\bm B,\bm C)_{(1)}\|_F \leq \|\bm A\|_2\|\ten S(\bm I,\bm B,\bm C)\|_F,
\] and then repeat the argument for the other modes. 
\end{proof}

\begin{lemma}
\label{lem:norm_bound_S}
For any $\ten S \in \mathbb{R}^{r\times r \times r}$, it holds that
\[
 \|\ten S(\ten S_{(1)},\ten S_{(2)},\ten S_{(3)})\|_F \geq \frac{1}{r^4}\|\ten S\|_F ^4
\]
\end{lemma}
\begin{proof}
Let $\bm u, \bm v, \bm w \in \mathbb{R}^{r}$ be unit vectors such that $\ten S(\bm u,\bm v,\bm w) = \|\ten S\|_2$.
Then
\begin{align*}
\ten S_{(3)}\text{vec}(\bm u\otimes \bm v)&= \ten S(\bm u,\bm v,\bm I) = \|\ten S\|_2\bm w,\\
\ten S_{(2)}\text{vec}(\bm u\otimes \bm w)&= \ten S(\bm u,\bm I,\bm w) = \|\ten S\|_2\bm v,\\
\ten S_{(1)}\text{vec}(\bm v\otimes \bm w)&= \ten S(\bm I,\bm v,\bm w) = \|\ten S\|_2\bm u.
\end{align*}
Then
\begin{align*}
\|\ten S(\ten S_{(1)},\ten S_{(2)},\ten S_{(3)})\|_2 &\geq
\ten S(\|\ten S\|_2\bm u, \|\ten S\|_2\bm v,\|\ten S\|_2\bm w) 
= \|\ten S\|_2^3\ten S(\bm u,\bm v,\bm w)= \|\ten S\|_2^4
\end{align*}
The result then follows from the norm inequality $\|\ten S\|_F\leq r\|\ten S\|_2$.
\end{proof}

Now we are ready to prove Lemma~\ref{lem:bounded_sublevel}:

\begin{proof}[Proof of Lemma~\ref{lem:bounded_sublevel}]
Assume that $\Gamma \geq f(\tup)$. 
From $L$, we have
\begin{align*}
\sqrt{\Gamma} &\geq \|\point-\ten T\|_F\\
&\geq \|\point\|_F - \|\ten T\|_F,
\end{align*}
 so that $\|\point\|_F \leq \sqrt{\Gamma} + \|\ten T\|_F$.
Next, define the following for $i = 1,2,3$: $d_i(\bm X,\ten S) = \bm X\bm X^\top - \ten S_{(i)}\ten S_{(i)}^\top$. 
Note that $d_1(\bm A,\ten S), d_2(\bm B,\ten S), d_3(\bm C,\ten S)$ are each bounded above in norm by $\Gamma^{1/4}$.
We have
\begin{align*}
\|\point\|_F^2 &= \langle \point,\point\rangle\\
&= \langle \ten S(\bm A \bm A^\top,\bm B\bm B^\top ,\bm C\bm C^\top ),\ten S\rangle\\
&= \langle \ten S(\ten S_{(1)}\ten S_{(1)}^\top,\ten S_{(2)}\ten S_{(2)}^\top,\ten S_{(3)}\ten S_{(3)}^\top), \ten S\rangle + g(\tup)\\
&=  \|\ten S(\ten S_{(1)},\ten S_{(2)},\ten S_{(3)})\|_F^2 + g(\tup),
\end{align*}
where $g(\tup )$ is a sum of the seven remainder terms of the form
\begin{align}
\label{eq:r1}
&\langle \ten S(d_1(\bm A,\ten S), \ten S_{(2)}\ten S_{(2)}^\top,\ten S_{(3)}\ten S_{(3)}^\top),\ten S \rangle\\
\label{eq:r2}
&\langle \ten S(d_1(\bm A,\ten S), d_2(\bm B,\ten S), \ten S_{(3)}\ten S_{(3)}^\top),\ten S\rangle\\
\label{eq:r3}
&\langle \ten S(d_1(\bm A,\ten S),d_2(\bm B,\ten S),d_3(\bm C,\ten S)),\ten S\rangle
\end{align}
There are three terms of type (\ref{eq:r1}), and each can be bounded below using Cauchy-Schwarz and Lemma \ref{lemma:submult} as follows: 
\[
\langle \ten S(d_1(\bm A,\ten S), \ten S_{(2)}\ten S_{(2)}^\top,\ten S_{(3)}\ten S_{(3)}^\top),\ten S \rangle \geq -\|\ten S\|_F^6\|d_1(\bm A,\ten S)\|_F \geq -\Gamma^{1/4}\|\ten S\|_F^6.
\]
Similarly, we have
\begin{align*}
\langle \ten S(d_1(\bm A,\ten S), d_2(\bm B,\ten S), \ten S_{(3)}\ten S_{(3)}^\top),\ten S\rangle &\geq -\Gamma^{1/2}\|\ten S\|_F^4,\\
\langle \ten S(d_1(\bm A,\ten S),d_2(\bm B,\ten S),d_3(\bm C,\ten S)),\ten S\rangle &\geq -\Gamma^{3/4}\|\ten S\|_F^2.
\end{align*}
Putting this together and applying Lemma \ref{lem:norm_bound_S}, we have that 
\[
\frac{1}{r^8}\|\ten S\|_F^8 - 3\Gamma^{1/4}\|\ten S\|_F^6 - 3\Gamma^{1/2}\|\ten S\|_F^4 - \Gamma^{3/4}\|\ten S\|_F^2 \leq (\sqrt{\Gamma}+\|\ten T\|_F)^2,
\]
which means that $\|\ten S\|_F$ must be bounded by $\so((\Gamma+1)^{1/8})$.
From $R$, we have 
\begin{align*}
\left( \frac{\Gamma}{\lambda}\right)^{1/4} + \|\ten S\|_F^2 
 \geq  \|\bm A\bm A^\top - \ten S_{(1)}\ten S_{(1)}^\top\|_F+ \|\ten S_{(1)}\ten S_{(1)}^\top\|_F \geq \|\bm A \bm A^\top\|_F,
\end{align*}
so $\|\bm A\|_F$ is bounded by $\so((\Gamma+1)^{1/8})$. We bound $\bm B$ and $\bm C$ similarly.
\end{proof}


\subsection{Main step: making local improvements}

In order to prove Theorem~\ref{thm:robust}, we rely on the following main lemma:
\begin{lemma}\label{lem:robustmain}
In the same setting as Theorem~\ref{thm:robust}, there exist positive constants $q_1, q_2$,$\tau_1 = \stheta(\epsilon^{q_1})$, $\tau_2 = \stheta(\epsilon^{q_2})$, such that for any point $\ten S,\bm A,\bm B, \bm C$ where $\epsilon < f(\ten S,\bm A,\bm B, \bm C) < O(1)$, one of the following is true:
\begin{enumerate}
\item $\|\nabla f(\tup)\| \ge \tau_1$,
\item $\lambda_{min}(\nabla^2 f(\tup)) \le -\tau_2$,
\item With constant probability, Algorithm \ref{alg:add_direction} constructs a direction of improvement that improves the function value by $\mbox{poly}(\epsilon)$. 
\end{enumerate}
\end{lemma}
Algorithm \ref{alg:add_direction} uses notation that we specify in the paragraphs below.
\begin{algorithm}
\begin{algorithmic}
\REQUIRE matrices $\bm A, \bm B, \bm C$, threshold $\sigma$, subspace indicator $(i,j,k) \in \{1,2\}^3$
\STATE Compute the subspaces $U_{1,i}, U_{2,j}, U_{3,k}, V_{1,i}, V_{2,j}, V_{3,k}$
\STATE Sample unit vectors $\bm a, \bm b, \bm c$ uniformly from $U_{1,i}, U_{2,j}, U_{3,k}$
\STATE {\bf if} $i=1$ {\bf then} $\bm u' = (\bm A_1^\top)^+\bm a$; {\bf else} Randomly sample nonzero $\bm u' \in V_{1,2}$
\STATE {\bf if} $j=1$ {\bf then} $\bm v' = (\bm B_1^\top)^+\bm b$; {\bf else} Randomly sample nonzero $\bm v' \in V_{2,2}$
\STATE {\bf if} $k=1$ {\bf then} $\bm w' = (\bm C_1^\top)^+\bm c$; {\bf else} Randomly sample nonzero $\bm w' \in V_{3,2}$
\STATE Return $\bm a, \bm b, \bm c, \bm u'/\|\bm u'\|_2, \bm v'/\|\bm v'\|_2, \bm w'/\|\bm w'\|_2$
\end{algorithmic}
\caption{Sampling algorithm for adding missing directions\label{alg:add_direction}}\end{algorithm}
The proof of this lemma has similar steps to the proof of Theorem~\ref{thm:exact}. However, it is more complicated because we are not looking at exact local minima. We give the details of these steps in the following subsections. A key parameter that we will use is a bound $\tau$ on the regularizer. We will consider different cases when $R(\tup) \ge \tau$ and when $R(\tup) \le \tau$. All of our other parameters (including $\tau_1,\tau_2,\epsilon$) will be polynomials in $\tau$. 

For the analysis, it is useful to consider $\point$ and $\ten T$ projected onto various subspaces of interest.
To this end, we introduce the following notation. 
Let $\sigma > 0$ be a threshold that we will specify later. For matrix $\bm A$, we let $V_{1,1}$ and $U_{1,1}$ denote the spaces spanned by the left/right singular vectors of $\bm A$ with singular value greater than $\sigma$, and let $V_{1,2} = V_{1,1}^\perp$, $U_{1,2} = U_{1,1}^\perp$. We can then write 
$\bm A = \bm A_1 + \bm A_2$, where $\bm A_1 = \mbox{Proj}_{V_{1,1}} \bm A$ contains the larger singular vectors and $\bm A_2 = \mbox{Proj}_{V_{1,2}}\bm A$ contains the smaller singular vectors. Let $\bm P_1$ be the orthogonal projection onto the column-space of $\ten T_{(1)}$ and define $\bm A_3 = \bm A(\bm I - \bm P_1),$  the projection of $\bm A$ onto directions that are unrelated to the true tensor. Similarly, we define $U_{2,1}, U_{2,2}, V_{2,1}, V_{2,2}, \bm P_2, \bm B_1, \bm B_2, \bm B_3$ for matrix $\bm B$ and $U_{3,1}, U_{3,2}, V_{3,1}, V_{3,2}, \bm P_3, \bm C_1, \bm C_2, \bm C_3$ for matrix $\bm C$.
Define $\ten S_{i,j,k} = \ten S( \mbox{Proj}_{V_{1,i}}, \mbox{Proj}_{V_{2,j}}, \mbox{Proj}_{V_{3,k}})$ 
and $\ten T_{i,j,k} = \ten T(\mbox{Proj}_{U_{1,i}}, \mbox{Proj}_{U_{2,j}}, \mbox{Proj}_{U_{3,k}})$.
We can decompose the tensor loss as
\[
\|\point - \ten T\|_F^2 = \sum_{i,j,k \in \{1,2\}}\|\ten S_{i,j,k}(\bm A_i, \bm B_j, \bm C_k) - \ten T_{i,j,k}\|_F^2.
\]
Our analysis shows how to decrease the objective function if the regularizer or any one of the terms in the right-hand sum is sufficiently large. In particular, after finding a second-order stationary point, the only terms in this sum that may be large are when at least two of $i, j, k$ are equal to $2$. In this case, Algorithm \ref{alg:add_direction} can be used to make further progress toward a local minimum.

\vspace*{-0.2in}
\subsection{Decreasing the Regularizer}

We first show if the regularizer is large, then the gradient is large. This is very similar to Lemma~\ref{lem:nonzeroreg}.
\begin{lemma}
\label{lem:decrease_regularizer}
If $R(\tup) \geq \tau$, then $\|\nabla f(\tup)\|_F \geq 4\lambda\tau/K$.
\end{lemma}
\begin{proof}
By assumption, $\reg(\tup) \geq \tau^{1/2}$, and we have$
\|\nabla R\|_F = \|2\reg\nabla \reg\|_F \geq 2\tau^{1/2}\|\nabla \reg\|_F.
$
By Lemma \ref{lem:reg_gradient} and the Cauchy-Schwarz inequality,  we have that
\[
\|\nabla \reg\|_F  \geq \frac{1}{2K}\|\nabla \reg\|_F\|(\tup)\|_F \geq \frac{1}{2K}\langle \nabla \reg, (\tup)\rangle = \frac{2\reg}{K}.
\]
Then $\|\nabla R\|_F \geq 4\tau/K$. Since $\|\nabla f\|_F^2 = \|\lambda \nabla R\|_F^2 + \|\nabla L\|_F^2,$ we are done.
\end{proof}

\subsection{Removing Extraneous Directions}
We show that if the projection $\bm A_3$ in the incorrect subspace is large, the gradient must be large so the point cannot be a local minimum. 
\begin{lemma}
\label{lem:removing_directions}
Let $\gamma = \stheta(\tau^{1/48})$.
If $R(\tup) < \tau$ and $\|\bm A_3\|_F \geq \gamma$, then
\[
\|\nabla f(\tup)\|_F = \somega(\tau^{1/6}).
\]
\end{lemma}
\begin{proof}
Set $\gamma = C\tau^{1/48}$, where we choose $C$ to be a constant such that
\[
\gamma^2 \geq \max\left(r(\tau^{1/24} +\tau^{1/4}), r^4(4K^6\tau^{1/8}+K^4\tau^{3/8})\right).
\]
This particular definition allows us to simplify inequality \eqref{ineq:SA2_total} below.
Consider the direction $\Delta \bm A =  - \bm A_3$. 
When we step in this direction, the first-order perturbation of $L(\ten S, \bm A + \epsilon \Delta \bm A, \bm B, \bm C)$ is $-2\epsilon\|\ten S(\bm A_3,\bm B,\bm C)\|_F^2$ (a simple calculation).
For the regularizer, 
observe that since $\bm A_3 = \bm A(\bm I - \bm P_3)$, we have $(\bm A+\epsilon \Delta \bm A)(\bm A+\epsilon\Delta \bm A)^\top = \bm A\bm A^\top - (2\epsilon-\epsilon^2) \bm A_3\bm A_3^\top.$ 
Hence the first-order perturbation of $R$ is
\[
8\epsilon \lambda \reg(\tup)\langle \bm A\bm A^\top - \ten S_{(1)}\ten S_{(1)}^\top,\bm A_3\bm A_3^\top\rangle \leq 8\epsilon\lambda\tau^{3/4}\|\bm A_3\|_F^2.
\]
Intuitively, we will show that the first-order decrease in $L$ is greater than the first-order increase in $R$, so that $\Delta \bm A$  is aligned negatively with $\nabla_{\bm A} f$. 

First, through very similar arguments to those found in the proof of Lemmas \ref{lem:bounded_sublevel} and \ref{lem:norm_bound_S}, we have that
\begin{equation}
\label{ineq:perturb1}
\|\ten S(\bm A_3,\bm B,\bm C)\|_F^2 \geq \|\ten S(\bm A_3,\ten S_{(2)}, \ten S_{(3)})\|_F^2 - 2\tau^{1/4}K^6 - \tau^{1/2} K^4
\end{equation}
and if we set $\bm u$ to be the top left singular vector of $\bm A_3$, 
\begin{align}
\label{ineq:perturb2}
\|\ten S(\bm A_3,\ten S_{(2)},\ten S_{(3)})\|_F &\geq \frac{1}{r^2}\|\ten S(\bm u,\bm I,\bm I)\|_F^3\|\bm A_3\|_F\\
\|\ten S(\bm u,\bm I,\bm I)\|_F^2 &= \bm u^\top \ten S_{(1)}\ten S_{(1)}^\top \bm u \nonumber \\
&= \bm u^\top \bm A\bm A^\top \bm u + \bm u^\top (\ten S_{(1)}\ten S_{(1)} - \bm A\bm A^\top )\bm u \nonumber \\
&
\label{ineq:perturb3}
\geq \frac{1}{r}\|\bm A_3\|_F^2 -  \tau^{1/4}.
\end{align}
Combining inequalities \eqref{ineq:perturb1}, \eqref{ineq:perturb2}, and \eqref{ineq:perturb3}, we have
\begin{equation}
\label{ineq:SA2_total}
\|\ten S(\bm A_3,\bm B,\bm C)\|_F^2 \geq \frac{1}{r^4}\|\bm A_3\|_F^2\left(\frac{1}{r}\|\bm A_3\|_F^2 - \tau^{1/4}\right)^3  - 2\tau^{1/4}K^6 - \tau^{1/2} K^4
\end{equation}

Using the assumption  $\|\bm A_3\|_F\geq \gamma$ and the choice of $\gamma$, we can simplify inequality \eqref{ineq:SA2_total} to 
$\|\ten S(\bm A_3,\bm B,\bm C)\|_F^2 \geq \frac{\tau^{1/8}}{2r^4}\|\bm A_3\|_F^2$. Now using the fact that $\lambda = 1/16r^4$ and $\tau ^{1/8} > \tau^{3/4}$, 
we have $\frac{\tau^{1/8}}{2r^4}\|\bm A_3\|_F^2 > 8\lambda\tau^{3/4}\|\bm A_3\|_F^2$.
Thus, we see that the first-order decrease in $L$ is greater than the first-order increase in $R$, and the overall first-order decrease in $f$ is $\somega(\tau^{1/8}\|A_3\|_F^2/2r^4)$.
The Taylor expansion of $f$ implies that  $|\langle \Delta \bm A, \nabla_{\bm A} f\rangle| \ge \frac{\tau^{1/8}}{2r^4}\|\bm A_3\|_F^2$, so that
\[
\|\nabla f\|_F \ge |\langle \Delta \bm A, \nabla_{\bm A} f\rangle| / \|\Delta \bm A\|_F  
= \somega(\tau^{1/8}\gamma) = \somega(\tau^{1/6}),
\]
which provides the desired bound on $\|\nabla f\|_F$.
\end{proof}

From this point forward, set $\sigma = \sqrt{\gamma}$.
An important consequence of the fact that $\bm A_3$ is small is that if $\ten T_{2,1,1}$ is large enough, then $\bm A_1$ must be rank deficient. This rank deficiency allows us to readily contruct a direction of improvement when we are near a saddle point corresponding to a single missing direction. This is also true when we are near saddle points corresponding to two or three missing directions; see section \ref{subsec:missing_dirs}.

To prove the rank deficiency, we use subspace perturbation bounds. 
The technical tool we use here is Wedin's Theorem~\citep{wedin1972perturbation,stewart1998perturbation}.
\begin{theorem*}[Wedin's Theorem, adapted from~\cite{stewart1998perturbation}]
Let $\tilde {\bm A}, \bm A, \bm E \in \mathbb{R}^{d\times r}$ with $d \geq r$ and $\tilde{\bm A} = \bm A + \bm E$.
Write the singular value decompositions
\[
\bm A = \begin{pmatrix}
\bm U_1 &\bm U_2
\end{pmatrix}
\begin{pmatrix}
\bm \Sigma\\
\bm 0
\end{pmatrix}
\bm V^\top
\qquad 
\tilde{\bm A} = \begin{pmatrix}
\tilde{\bm U_1} &\tilde{\bm U_2}
\end{pmatrix}
\begin{pmatrix}
\tilde{\bm \Sigma}\\
\bm 0
\end{pmatrix}
\tilde{\bm V}^\top
\]
Let $\bm \Theta$ and $\bm \Phi$ denote the matrices of principal angles between the column spans of $\bm U_1, \tilde{\bm U}_1$ and
$\bm V, \tilde{\bm V}$, respectively.
If there exists some $\delta > 0$ such that $\min \sigma(\tilde{\bm \Sigma}) \geq \delta$, then
\[
\sqrt{\|\sin \bm \Theta\|_F^2 + \|\sin \bm \Phi\|_F^2} \leq \frac{\sqrt{2}\|\bm E\|_F}{\delta}.
\]
\end{theorem*}

\begin{lemma}
\label{lem:perturbation}
Let $\bm M \in \mathbb{R}^{r\times d}$, and let $\bm M = \bm M_1 + \bm M_2$, where $\text{rank}(\bm M) = \text{rank}(\bm M_1) = r$.
Let $\bm P, \bm P_1 \in \mathbb{R}^{d\times d}$ be the orthogonal projections onto the row spans of $\bm M$ and $\bm M_1$, 
respectively.
Let $\sigma$ be the smallest nontrivial singular value of $\bm M$.
Then
\[
\|\bm P - \bm P_1\|_F \leq \frac{2\|\bm M_2\|_F}{\sigma}.
\]
\end{lemma}
\begin{proof}
This is a corollary of Wedin's Theorem.
Set $\bm A = \bm M_1^\top$, $\tilde{\bm A} = \bm M^\top$, and $\bm E = \bm M_2^\top$.
Note that $\bm A$ and $\tilde{\bm A}$ have full row rank, so we have the SVDs
\[
\bm A = \begin{pmatrix}
\bm U_1 &\bm U_2
\end{pmatrix}
\begin{pmatrix}
\bm \Sigma\\
\bm 0
\end{pmatrix}
\bm V^\top
\qquad 
\tilde{\bm A} = \begin{pmatrix}
\tilde{\bm U_1} &\tilde{\bm U_2}
\end{pmatrix}
\begin{pmatrix}
\tilde{\bm \Sigma}\\
\bm 0
\end{pmatrix}
\tilde{\bm V}^\top
\]
where $\bm V, \tilde{\bm  V}$ are $r\times r$ orthogonal matrices, $\bm \Sigma, \tilde{\bm\Sigma}$ are $r \times r$, $\bm U_1, \tilde{\bm  U}_1$ are $d\times r$, and 
$\bm U_2, \tilde{\bm U}_2$ are $d\times (d-r)$.
Since $\bm V$ and $\tilde{\bm V}$ have the same column spans, we have that $\sin \bm \Phi =\bm 0$.
Further, it is a fact that $\|\bm P-\bm P_1\|_F = \sqrt{2}\|\sin \bm\Theta\|_F$.
By assumption, $\min\sigma(\tilde {\bm\Sigma}) = \sigma$.
Then Wedin's Theorem states that
\[
\sqrt{\|\sin \bm\Theta \|_F^2 + \|\sin \bm\Phi\|_F^2} \leq \frac{\sqrt{2}\|\bm E\|_F}{\sigma},
\]
and our result follows immediately.
\end{proof}

\begin{lemma}
\label{lem:rank_bound}
Let $\bm P$ be the orthogonal projection onto the row-span of $\bm A_1$.
If $\text{rank}(\bm A_1) =r$ and $\|\bm A_3\|_F \leq \gamma$, then $\|\ten T(\bm I - \bm P, \bm I, \bm I)\|_F < 2K\sqrt{\gamma}$.
In particular, if any of the $T_{1,j,k}$ ($j,k=1,2$) is large, the rank of $\bm A_1$ must be less than $r$.
\end{lemma}
\begin{proof}
Recall we set $\sigma = \sqrt{\gamma}$.
Let $\bm P_1$ be the orthogonal projection onto the column-span of $\ten T_{(1)}$.
Write $\bm A_{1,1} = \bm A_1\bm P_1$, $\bm A_{1,2} = \bm A_1(\bm I - \bm P_1)$.
Observe that $\|\bm A_{1,2}\|_F \leq \|\bm A_3\|_F \leq \gamma < \sigma$.
Note that $\|\bm A_1 - \bm A_{1,1}\|_F = \|\bm A_{1,2}\|_F < \sigma$, which means that $\text{rank}(\bm A_{1,1}) = r$,
since $\bm A_1$ has distance at least $\sigma$ to the closest lower-rank matrix.

Since $\bm A_{1,1}$ has rank $r$, its rows form a basis for the column-span of $\ten T_{(1)}$,
and so $\bm P_1$ is also the orthogonal projection onto the row-span of $\bm A_{1,1}$. 
Then
\begin{align*}
\|\ten T(\bm I - \bm P, \bm I, \bm I)\|_F &= \|\ten T(\bm P_1 - \bm P, \bm I, \bm I)\|_F\\
&\leq \|\ten T\|_F\|\bm P_1 - \bm P\|_F\\
&\leq K\frac{2\|\bm A_{1,2}\|_F}{\sigma}\\
&<2K\sqrt{\gamma},
\end{align*} 
where the penultimate line follows from Lemma \ref{lem:perturbation}.
\end{proof}

\subsection{Improving $S$}
Unlike the proof of Theorem~\ref{thm:exact}, we will first focus on the simple case of improving the core tensor $\ten S$. Note that here we only try to make sure we get close to $\ten T_{1,1,1}$ as the components $\bm A,\bm B, \bm C$ may still be missing directions.
\begin{lemma}
\label{lem:improve_S}
Set $\kappa_0 = \sqrt{\gamma}$.
Assume $R(\tup) < \tau$. Then
\[
\|\ten T_{1,1,1} - \ten S(\bm A_1, \bm B_1, \bm C_1)\|_F > \kappa_0\,\, \Rightarrow\,\, \|\nabla f(\tup)\|_F = \Omega(\gamma^{2.5}).
\]
\end{lemma}
\begin{proof}
Define $\ten S^* = \ten T(\bm A_1^+, \bm B_1^+,\bm C_1^+)$, so that $\ten S^*(\bm A_1,\bm B_1,\bm C_1) = \ten T_{1,1,1}$.
We consider the direction $\Delta \ten S = \ten S^* - \ten S_{1,1,1}$.
Observe that $\Delta \ten S(\bm A, \bm B, \bm C) =\ten T_{1,1,1} - \ten S(\bm A_1, \bm B_1, \bm C_1)$.
We can write
\begin{align*}
\ten S(\bm A, \bm B, \bm C) - \ten T = \sum_{i,j,k \in \{1,2\}} \ten S(\bm A_i, \bm B_j, \bm C_k) - \ten T_{i,j,k},
\end{align*}
and this is a sum of mutually orthogonal tensors. 
Hence, the 
the first-order perturbation of $L(\ten S + \epsilon \Delta \ten S, \bm A, \bm B, \bm C)$ is
\begin{align*}
 2\langle \point - \ten T, \Delta \ten S(\bm A, \bm B, \bm C) \rangle &= -2 \|\Delta\ten S(\bm A, \bm B,\bm C)\|_F^2.
\end{align*}

The first-order perturbation in the regularizer $\langle \nabla_{\ten S} R, \Delta \ten S\rangle$ is bounded by $O(\tau^{3/4}\sigma^{-3}) = o(\sigma)$, since $\|\ten S^*\|_F = O(\sigma^{-3})$.
Therefore, the decrease in the tensor loss dominates all other first-order perturbations, so we have a viable direction of improvement.
In particular, by moving in direction $\epsilon \Delta \ten S$, 
we decrease the objective function by
\[
\epsilon\cdot \Omega(\|\ten T_{1,1,1} - \ten S(\bm A_1, \bm B_2, \bm C_1)\|_F^2) = \Omega(\epsilon\kappa_0^2).
\]
The direction of movement has norm bounded by 
\[\|\ten T_{1,1,1}\|_F \|\bm A_1^+\|_2\|\bm B_1^+\|_2\|\bm C_1^+\|_2 \le K\sigma^{-3}.\] By Cauchy-Schwarz, the gradient has norm at least $\Omega(\kappa_0^2)\times \sigma^3 = \Omega(\gamma^{5/2}).$
\end{proof}

\subsection{Adding missing directions}
\label{subsec:missing_dirs}
Finally, we try to add missing directions to $\bm A, \bm B, \bm C$. As before we separate the cases into missing 1, 2 and 3 directions. This first case (missing one direction) is easy as it is a normal saddle point with negative Hessian.
\begin{lemma}
\label{lem:add_1_direction}
Set $\kappa_1 = 2K\sigma^{3/4}$.
Assume $R(\tup) < \tau$ and $\|\bm A_3\|_F$, $\|\bm B_3\|_F$, and $\|\bm C_3\|_F$ are all less than $\gamma$.
If $\|\ten T_{2,1,1}\|_2 \geq \kappa_1$,
then $\nabla^2 f$ has a negative eigenvalue of at most $-\Omega(\sigma^{15/4})$. 
\end{lemma}
\begin{proof}
Since $\kappa_1 > 2K\sqrt{\gamma}$, by Lemma \ref{lem:rank_bound}, we have $\text{rank}(\bm A_1) < r$.

By assumption, there exist unit vectors $\bm a \in U_{1,2}$, $\bm b \in U_{2,1}$, and $\bm c \in U_{3,1}$ such that 
$\ten T(\bm a,\bm b,\bm c) \geq \kappa_1$.
Take unit vectors $\bm u, \bm v, \bm w \in \mathbb{R}^r$ such that $\bm A_1^\top \bm u = \bm 0$,  $\bm B_1^\top \bm v =  \alpha_1 \bm b$, $\bm B_2^\top \bm v =\bm 0$, $\bm C_1^\top \bm w = \alpha_2 \bm c$, and $\bm C_2^\top \bm w = \bm 0$, where
$\alpha_i \geq \sigma$ for $i=1,2$.
In this situation, we are near a second-order saddle point, so we seek to demonstrate a direction with sufficient negative curvature in the objective function.
To this end, define
\begin{align*}
\Delta \bm A = \sigma \bm u\bm a^\top \qquad \Delta \ten S = \bm u\otimes \bm v\otimes \bm w.
\end{align*}
For a step size $\epsilon > 0$, our source of improvement in the tensor loss comes from the second-order perturbation of $L$ in this direction. We aim to compare the second-order decrease in $L$ against the second-order increases in $L$ and $R$. 
The second-order perturbation in the tensor loss $\nabla^2 L$ applied to $(\Delta \ten S, \Delta \bm A, \bm 0, \bm 0)$ is
\[
2\langle \diff, \Delta \ten S(\Delta \bm A, \bm B, \bm C)\rangle + \|\Delta \point + \ten S(\Delta \bm A, \bm B, \bm C)\|^2
\]
The magnitude of \emph{decrease} in this perturbation is given by
\begin{align*}
\langle \ten T,\Delta \ten S(\Delta \bm A, \bm B, \bm C)\rangle &= \sigma\ten T(\bm a,\alpha_1 \bm b , \alpha_2 \bm c )\\
&= \sigma \alpha_1\alpha_2\ten T(\bm a,\bm b,\bm c)\\
&\geq  \sigma \alpha_1\alpha_2\kappa_1.
\end{align*}
To bound the magnitude of the \emph{increase}, observe that 
\begin{align*}
\|\bm B\bm B^\top \bm v\|_F &= \| \alpha_1\bm B_1\bm b\|_F \leq \alpha_1K; \quad \|\bm C\bm C^\top \bm w\|_F \leq \alpha_2K 
\end{align*}
Then we have
\begin{align}
\langle \point ,\Delta \ten S(\Delta \bm A, \bm B, \bm C)\rangle &= \langle \point, \sigma \bm a \otimes  \bm B^\top \bm v \otimes \bm C^\top \bm w\rangle \nonumber\\
&= \sigma \ten S(\bm A\bm a, \bm B\bm B^\top \bm v, \bm C\bm C^\top \bm w) \nonumber\\
\label{ineq:bound1}
&\leq \sigma^2\alpha_1\alpha_2K^3 
\end{align}
Additionally,  
\begin{align*}
\|\Delta \point\|_F^2 &= \|\bm A_2^\top \bm u\otimes \alpha_1 \bm b \otimes \alpha_2 \bm c \|_F^2\\ 
&\leq \sigma^2\alpha_1^2\alpha_2^2\\
\|\ten S_{(1)}^\top \bm u\|_F^2 &= \bm u^\top(\ten S_{(1)}\ten S_{(1)}^\top  - \bm A\bm A^\top )\bm u + \bm u^\top\bm  A\bm A^\top \bm  u\\ 
&\leq \tau^{1/4} + \sigma^2\\
 \|\ten S(\Delta \bm A,\bm B,\bm C)\|_F^2 &=\sigma^2\|\bm a\bm u^\top \ten S_{(1)}(\bm B \otimes \bm C)\|_F^2\\
&\leq \sigma^2 \|\ten S_{(1)}^\top \bm u\|_F^2\|\bm B\|_F^2\|\bm C\|_F^2\\
&\leq\sigma^2 K^4(\tau^{1/4}+\sigma^2)
\end{align*}
Putting this together, we bound $\| \Delta \point + \ten S( \Delta \bm A, \bm B, \bm C)\|_F^2$ above by
\begin{align}
\label{ineq:bound2}
\left(\sigma\alpha_1\alpha_2 + \sigma K^2\sqrt{\tau^{1/4}+\sigma^2}\right)^2
\end{align}
In light of the definition of $\kappa_1$ and inequalities (\ref{ineq:bound1}) and (\ref{ineq:bound2}), the second-order perturbation in $L$ is $-\Omega(\sigma\alpha_1\alpha_2\kappa_1)$, i.e. $L$ decreases to second-order in this direction.

Now we turn our attention to the regularizer. 
We need to show that the second-order increase in the regularizer doesn't overwhelm the decrease in $L$.
Note that the regularizer is degree 4 with respect to $\|\bm A\bm A^\top - \ten S_{(1)}\ten S_{(1)}^\top\|_F$ (and same terms for $\bm B$ and $\bm C$), so the second order derivatives have a quadratic term in $\|\bm A\bm A^\top - \ten S_{(1)}\ten S_{(1)}^\top\|_F$,
which is $O(\tau^{1/4}) = o(\sigma\alpha_1\alpha_2\kappa_1)$; higher-order terms are negligible in comparison. 

We've shown that the loss function decreases by at least $\Omega(\sigma\alpha_1\alpha_2\kappa_1)\cdot \epsilon^2$. Since our direction of improvement has constant norm, this implies that $\nabla^2 f$ has an eigenvalue that is smaller than $-\Omega(\sigma\alpha_1\alpha_2\kappa_1) = - \Omega(\sigma^{15/4})$. 
\end{proof}

Next, we need to deal with the high order saddle points. Here our main observation is that if we choose directions randomly in the correct subspace, then the perturbation is going to have a reasonable correlation with the residual tensor with constant probability. This is captured by the following anti-concentration property:
\begin{lemma}
\label{lem:anti}
Let $\ten X \in \mathbb{R}^{d_1\times d_2\times d_3}$, and let $\bm a\in\mathbb{R}^{d_1}, \bm b\in\mathbb{R}^{d_2}, \bm c\in\mathbb{R}^{d_3}$ be independent, uniformly distributed unit vectors.  There exist positive numbers $C_1 = \Omega(1/\sqrt{d_1d_2d_3}) = \somega(1), C_2 = \Omega(1)$ such that 
\[
\text{Pr}[|\ten X(\bm a,\bm b,\bm c)| \geq C_1\|\ten X\|_F] > C_2.
\]
\end{lemma}
\begin{proof}
Although our Algorithm \ref{alg:add_direction} for sampling missing directions only requires uniform unit vectors (from appropriate subspaces), we construct these vectors as normalized Gaussian vectors for this lemma in order to apply a Gaussian polynomial anti-concentration result (Theorem 8 in \cite{carbery2001distributional}).
As such, let $\bm a', \bm b', \bm c'$ be independent standard Gaussian random vectors (of appropriate dimension) and set $\bm a = \bm a'/\|\bm a'\|_2$, $\bm b = \bm b'/\|\bm b'\|_2$, and $\bm c = \bm c'/\|\bm c'\|_2$.
Note that there exists some constant $p> 0$ such that $\|\bm a'\|_2 \le 2\sqrt{d_1}$, $\|\bm b'\|_2\le 2\sqrt{d_2}$, and $\|\bm c'\| \le  2\sqrt{d_3}$ with probability at least $p$.
Next, note that $\mathbb{E} \ten X(\bm a', \bm b', \bm c') = 0$ and
\begin{align*}
\text{Var}[\ten X(\bm a', \bm b', \bm c')] &= \mathbb{E}(\ten X(\bm a', \bm b', \bm c')^2)\\
&= \mathbb{E}\langle \ten X(\bm a'\bm a'^\top, \bm b'\bm b'^\top, \bm c'\bm c'^\top),\ten X\rangle\\
&= \langle \ten X(\bm I, \bm I, \bm I),\ten X\rangle\\
&= \|\ten X\|_F^2.
\end{align*}
Now $\ten X(\bm a',\bm b',\bm c')/\|\ten X\|_F$ is a degree three polynomial function with unit variance, so the anti-concentration inequality
 implies that for any $\epsilon > 0$, 
\begin{align*}
\text{Pr}[|\ten X(\bm a',\bm b',\bm c')/\|\ten X\|_F| \leq \epsilon] \leq O(1)\epsilon^{1/3}.
\end{align*}
Simply choosing a constant $\epsilon$ and re-arranging terms completes the proof.
\end{proof}

Using this idea, when $\ten T_{2,2,1}$ is large, we show how to get a direction of improvement.

\begin{lemma}
\label{lem:add_2_directions}
Set  $\kappa_2 = 2K\sigma^{1/8}$.
Assume $R(\tup) < \tau$ and $\|\bm A_3\|_F$, $\|\bm B_3\|_F$, and $\|\bm C_3\|_F$ are all less than $\gamma$.
Further assume that $\|\ten T_{2,1,1}\|_2$, $\|\ten T_{1,2,1}\|_2$, and $\|\ten T_{1,1,2}\|_2$ are each less than $\kappa_1$.
Let $\bm a, \bm b, \bm c, \bm u, \bm v, \bm w$ be the output of Algorithm \ref{alg:add_direction} given the input $\bm A, \bm B, \bm C, \sigma, (2,2,1)$. 
Define the directions $\Delta \bm A = \bm u\bm a^\top$, $\Delta \bm B = \bm v \bm b^\top$, $\Delta \ten S = \bm u \otimes \bm v \otimes \bm w$.
If $\|\ten T_{2,2,1}\|_2 \geq \kappa_2$, then with constant probability, a step in these directions decreases the objective function 
by $\Omega^*(\sigma^{15/8})$.
\end{lemma}

\begin{proof}
First, observe that $\kappa_2 \geq 2K\sqrt{\gamma}$, which implies that $\text{rank}(\bm A_1) < r$ and $\text{rank}(\bm B_1)< r$ by Lemma \ref{lem:rank_bound}.
Set $\alpha$ such that $\alpha \bm  c = \bm C_1^\top \bm w$, and note that $\alpha \geq \sigma$.
Per lemma \ref{lem:anti}, with constant probability we have $|\ten T(\bm a,\bm b, \bm c)|$ is with some constant factor of $ \|\ten T_{2,2,1}\|_2$. Therefore, with constant probability,  $|\ten T(\bm a,\bm b, \bm c)| \geq C\kappa_2$ for some positive constant $C$. 

Observe that $p(\delta) := f(\ten S+\delta \Delta \ten S, \bm A+\delta \Delta \bm A, \bm B+\delta\Delta \bm B, \bm C)$ defines a degree $8$ polynomial in $\delta$.
Set $\delta = \sigma^{1/4}$.
For convenience, define the following expressions related to the perturbations of $L$:
\begin{align*}
L_0 &= \diff\\
L_1 &= \Delta \point + \ten S(\Delta \bm A, \bm B, \bm C) + \ten S(\bm A,\Delta \bm B, \bm C) \\
L_2 &= \Delta \ten S(\Delta \bm A, \bm B, \bm C) + \Delta \ten S(\bm A,\Delta \bm B, \bm C) +\ten S(\Delta \bm A, \Delta \bm B, \bm C)\\
L_3 &= \Delta\ten S(\Delta\bm A, \Delta\bm B,\bm C)
\end{align*}
We can upper bound each of these terms in norm, e.g.
\begin{align*}
\|L_1\|_F &= \| \bm A^\top \bm u \otimes \bm B^\top \bm v \otimes \bm C^\top \bm w + \ten S(\bm u\bm a^\top,\bm B,\bm C) + \ten S(\bm A,\bm v\bm b^\top,\bm C)\|_F\\
&\leq \sigma^2\alpha + 2K^2\sqrt{\tau^{1/4}+\sigma^2}\\
&= O(\alpha\sigma^2 + \sigma).
\end{align*}
Through similar calculations, we have $\|L_2\|_F = O(\alpha\sigma + \sigma)$ and $\|L_3\|_F = O(\alpha)$. 
The perturbation in the tensor loss is then
\begin{equation}
\label{eq:loss_perturb}
\delta^3 \langle  L_0,L_3\rangle + \delta \langle L_0,L_1\rangle + \delta^2\langle L_0,L_2\rangle + \|\delta L_1+ \delta^2 L_2+\delta^3 L_3\|_F^2.
\end{equation}
Here the first term is responsible for the decrease in tensor loss:
\begin{align*}
\delta^3\langle L_0,L_3\rangle &= \delta^3\alpha\langle \point - \ten T,\bm a \otimes \bm b \otimes \bm c\rangle\\
&\leq \delta^3\alpha(K^2\sigma^2 - \ten T(\bm a,\bm b,\bm c))\\
&= -\delta^3\alpha\Omega(\kappa_2). 
\end{align*}
For the other terms, we show that they are small enough so they will not cancel this improvement.
Observe that 
\begin{align*}
\langle L_0,L_1\rangle &= \langle L_0, \Delta\ten S(\bm A, \bm B, \bm C)\rangle + \langle L_0, \ten S(\bm u\bm a^\top, \bm B, \bm C) + \ten S(\bm A, \bm v\bm b^\top, \bm C)\rangle\\
&= O(\alpha \sigma^2) + O(\kappa_1\sigma). 
\end{align*}
The $O(\kappa_1\sigma)$ term appears because $\|\ten S_{2,1,1}(\bm A, \bm B, \bm C) - \ten T_{2,1,1}\|_F = O(\kappa_1)$, 
$\|\ten S_{1,2,1}(\bm A, \bm B, \bm C) - \ten T_{1,2,1}\|_F = O(\kappa_1)$.

For the next term, we note that $\langle L_0, L_2\rangle$ is a sum of three inner products, any two of which we can make nonpositive by flipping the sign of $\Delta \ten S$ and one of $\Delta \bm A$, $\Delta \bm B$ (doing so doesn't change the amount by which the tensor loss decreases). 
Hence, by design of Algorithm~\ref{alg:add_direction}, with constant probability we know that $\langle L_0,L_2\rangle \leq 0$. 
As a result, we know \eqref{eq:loss_perturb} is at most $-\delta^3 \alpha \Omega(\kappa_2)$. 

We now consider the perturbations of the regularizer.
Define the following terms:
\begin{align*}
\reg_{0,1} &= \bm A\bm A^\top - \ten S_{(1)}\ten S_{(1)}^\top,\quad \reg_{0,2} =\bm  B\bm B^\top - \ten S_{(2)}\ten S_{(2)}^\top,\quad \reg_{0,3} =\bm  C\bm C^\top - \ten S_{(3)}\ten S_{(3)}^\top\\
\reg_{1,1} &= \bm A\Delta\bm A^\top + \Delta\bm A\bm A^\top - \ten S_{(1)}(\Delta\ten S)_{(1)}^\top - (\Delta\ten S)_{(1)}\ten S_{(1)}^\top\\
\reg_{1,2} &= \bm B\Delta\bm B^\top + \Delta\bm B\bm B^\top - \ten S_{(2)}(\Delta\ten S)_{2)}^\top - (\Delta\ten S)_{(2)}\ten S_{(2)}^\top\\
\reg_{1,3} &= - \ten S_{(3)}(\Delta\ten S)_{(3)}^\top - (\Delta\ten S)_{(3)}\ten S_{(3)}^\top,\quad \reg_{2,3} = - (\Delta\ten S)_{(3)}(\Delta\ten S)_{(3)}^\top
\end{align*}
We bound these terms in norm as follows:
\begin{align*}
\|\reg_{1,1}\|_F &= \|\bm A\bm a\bm u^\top + \bm u\bm a^\top \bm A^\top - \bm u\ten S(\bm I,\bm v,\bm w)^\top - \ten S(\bm I,\bm v,\bm w)\bm u^\top\|_F\\
&\leq 2\|\bm A_2\|_F + 2\|\ten S(\bm I,\bm v,\bm w)\|_F\\
&\leq 2\sigma + 2\sqrt{\tau^{1/4}+\sigma^2}\\
&= O(\sigma).
\end{align*}
Likewise, $\|\reg_{1,2}\|_F = O(\sigma)$, $\|\reg_{1,3}\|_F = O(\sigma)$, and $\|\reg_{2,3}\| = O(1)$.
Also note that $\|\reg_{0,i}\|_F \leq \tau^{1/4}$ for $i=1,2,3$.
Using this, 
we have
\begin{align*}
\|\reg_{0,1} + \delta \reg_{1,1}\|_F &\leq O(\tau^{1/4}) +O(\delta\sigma)\\
\|\reg_{0,2} + \delta \reg_{1,2}\|_F &\leq O(\tau^{1/4}) +O(\delta\sigma)\\
\|\reg_{0,3} +  \delta \reg_{1,3} + \delta^2 \reg_{2,3}\|_F &\leq O(\tau^{1/4}) + O(\delta\sigma) + O(\delta^2).
\end{align*}
All of these terms are dominated by $O(\delta^2)$, and so the perturbed regularizer is bounded by $O(\delta^8) = O(\sigma^2) = o(\sigma^{15/8})$. Hence, we see that the decrease in the tensor loss dominates the increase in the regularizer, as desired.
\end{proof}

We next address the case where $\ten T_{2,2,2}$ is large, which corresponds to $\bm A_1, \bm B_1, \bm C_1$ being rank deficient.
\begin{lemma}
\label{lem:add_3_directions}
Set  $\kappa_3 = 2K\sigma^{1/2}$.
Assume $R(\tup) < \tau$ and $\|\bm A_3\|_F$, $\|\bm B_3\|_F$, and $\|\bm C_3\|_F$ are all less than $\gamma$.
Further assume that $\|\ten T_{2,1,1}\|_2$, $\|\ten T_{1,2,1}\|_2$, and $\|\ten T_{1,1,2}\|_2$ and each less than $\kappa_1$,
while $\|\ten T_{2,2,1}\|_2$, $\|\ten T_{2,1,2}\|_2$, and $\|\ten T_{1,2,2}\|_2$ are each less than $\kappa_2$.
Let $\bm a, \bm b, \bm c, \bm u, \bm v, \bm w$ be the output of Algorithm \ref{alg:add_direction} with input $\bm A, \bm B, \bm C, \sigma, (2,2,2)$.
Define the directions $\Delta \bm A = \bm u \bm a^\top$, $\Delta \bm B = \bm v \bm b^\top$, $\Delta \bm C = \bm w \bm c^\top$, $\Delta \ten S = \bm u \otimes \bm v \otimes \bm w$.
If $\|\ten T_{2,2,2}\|_2 \geq \kappa_3$, then with constant probability, a step in these directions decreases the objective function 
by $\Omega^*(\sigma^{3/4})$.
\end{lemma}
\begin{proof}
First observe that $\kappa_3 \geq 2K\sqrt{\gamma}$, which by Lemma \ref{lem:rank_bound} means that  $\text{rank}(\bm A_1)$, $\text{rank}(\bm B_1)$, and $\text{rank}(\bm C_1)$ are all strictly less than $r$.
Then  $\bm A_1$, $\bm B_1$, and $\bm C_1$ are all missing directions from the relevant subspaces of $\ten T$, 
and we are near a fourth-order saddle point.
By lemma \ref{lem:anti}, with constant probability, $|\ten T(\bm a,\bm b,\bm c)| > C\kappa_3$ for some positive constant $C$.

Again  let $p(\delta) = f(\ten S+\delta \Delta \ten S, \bm A+\delta \Delta \bm A, \bm B+\delta\Delta \bm B, \bm C+\delta\Delta \bm C)$,
and set $\delta = \sigma^{1/8}$.
As in the proof of lemma \ref{lem:add_2_directions}, let for $i=0,\ldots,4$, let $L_i$ denote the $i$-th order perturbation term in 
\[
(\ten S + \Delta \ten S)(\bm A + \Delta \bm A, \bm B + \Delta \bm B, \bm C + \Delta \bm C) - \ten T.
\]
We can upper bound each of these terms in norm, e.g.
\begin{align*}
\|L_1\|_F &= \|\bm A^\top \bm u \otimes \bm B^\top \bm v \otimes \bm C^\top \bm w + \ten S(\bm u\bm a^\top,\bm B,\bm C) + \ten S(\bm A,\bm v\bm b^\top,\bm C)\\ 
& + \ten S(\bm A,\bm B,\bm w\bm c^\top)\|_F\\
&\leq 8\sigma^3 + K^2(\|\ten S(\bm u,\bm I,\bm I)\|_F + \|\ten S(\bm I,\bm v,\bm I)\|_F+\|\ten S(\bm I,\bm I,\bm w)\|_F)\\
&\leq 8\sigma^3 + 3K^2\sqrt{\tau^{1/4}+2\sigma^2}\\
&= O(\sigma).
\end{align*}
Through similar calculations, we have $\|L_2\| = O(\sigma)$ and $\|L_3\| = O(\sigma)$. 
On the other hand, $\|L_4\| \leq 1$ and $\|L_0\|\leq 2K$.
The perturbation in the tensor loss is then
\begin{equation}
\label{eq:loss_perturb2}
\sum_{i,j = 0}^4 \langle L_i,L_j\rangle \delta^{i+j}
\end{equation}
The decrease in the tensor loss is due to the following term:
\begin{align*}
\delta^4\langle L_0,L_4\rangle &= \delta^4\langle \point - \ten T,\bm a \otimes \bm b \otimes \bm c\rangle\\
&\leq \delta^4( K\sigma^3 - \ten T(\bm a,\bm b,\bm c))\\
&= -\delta^4\Omega(\kappa_3)\\
&= -\Omega(\sigma^{3/4}).
\end{align*}
By a simple Cauchy-Schwarz bound, the other perturbation terms in (\ref{eq:loss_perturb2})
are all bounded by $O(\sigma+ \delta^8) = O(\sigma) = o(\sigma^{3/4})$. 

Now we analyze the perturbations of the regularizer.
As before, define the terms
\begin{align*}
\reg_{0,1} &= \bm A\bm A^\top - \ten S_{(1)}\ten S_{(1)}^\top,\quad \reg_{0,2} =\bm  B\bm B^\top - \ten S_{(2)}\ten S_{(2)}^\top,\quad \reg_{0,3} =\bm  C\bm C^\top - \ten S_{(3)}\ten S_{(3)}^\top\\
\reg_{1,1} &= \bm A\Delta\bm A^\top + \Delta\bm A\bm A^\top - \ten S_{(1)}(\Delta\ten S)_{(1)}^\top - (\Delta\ten S)_{(1)}\ten S_{(1)}^\top\\
\reg_{1,2} &= \bm B\Delta\bm B^\top + \Delta\bm B\bm B^\top - \ten S_{(2)}(\Delta\ten S)_{2)}^\top - (\Delta\ten S)_{(2)}\ten S_{(2)}^\top\\
\reg_{1,3} &= \bm C\Delta\bm C^\top + \Delta\bm C\bm C^\top - \ten S_{(3)}(\Delta\ten S)_{(3)}^\top - (\Delta\ten S)_{(3)}\ten S_{(3)}^\top
\end{align*}
We bound these terms in norm as follows:
\begin{align*}
\|\reg_{1,1}\|_F &= \|\bm A\bm a\bm u^\top + \bm u\bm a^\top \bm A^\top - \bm u\ten S(\bm I,\bm v,\bm w)^\top - \ten S(\bm I,\bm v,\bm w)\bm u^\top\|_F\\
&\leq 2\|\bm A_2\|_F + 2\|\ten S(\bm I,\bm v,\bm w)\|_F\\
&\leq 2\sigma + 2\sqrt{\tau^{1/4}+\sigma^2}\\
&= O(\sigma).
\end{align*}
Likewise, $\|\reg_{1,i}\|_F = O(\sigma)$ for $i=2,3$, and of course $\|\reg_{0,i}\|_F \leq \tau^{1/4}$ for $i=1,2,3$.
Again,
\begin{align*}
\|\reg_{0,i} + \delta \reg_{1,i}\|_F \leq O(\tau^{1/4}) + O(\delta\sigma)
\end{align*}
and using this, we can bound the perturbed regularizer as $O(\delta^4\sigma^4) = o(\sigma^{3/4})$.
Hence, the decrease in the tensor loss dominates all other perturbations, and we improve the objective function by $\Omega(\sigma^{3/4})$.
\end{proof}

\subsection{Algorithm Description and Proof of Main Theorem}
\label{sec:algdesc}

Before sketching the algorithm we will first prove Lemma~\ref{lem:robustmain} and explain some of the parameter choices.
Refer to Table \ref{tab:parameters} for a list of the numerical quantities that were introduced.

\begin{proof}[Proof of Lemma~\ref{lem:robustmain}]
Set $\tau$ small enough so that $\kappa_0$, $d\kappa_1 + K^3\sigma$, $d\kappa_2+K^2\sigma^2$, $d\kappa_3+K\sigma^3 < \sqrt{\epsilon}/4$ and $\tau < \epsilon/2$.   We then set $\tau_1 = \stheta(\tau)$ from Lemma~\ref{lem:decrease_regularizer} and $\tau_2 = \Theta(\sigma^{15/4})$ from Lemma~\ref{lem:add_1_direction}.

Now assume that conditions (1), (2), and (3) from the statement of the Lemma fail to hold.  
We seek to show that $f(\tup) < \epsilon$.
By Lemma \ref{lem:decrease_regularizer} and our choice of $\tau_1$, we have that $R(\tup) < \tau \leq \epsilon/2$.
By Lemma \ref{lem:removing_directions}, we have that $\|\bm A_3\|_F, \|\bm B_3\|_F, \|\bm C_3\|_F$ are all less than $\gamma$.
By Lemma \ref{lem:improve_S}, we have that $\|\ten T_{1,1,1} - \point\|_F \leq \kappa_0$.
By Lemma \ref{lem:add_1_direction}, we have that $\|\ten T_{i,j,k}\|_2 < \kappa_1$ for $(i,j,k) \in \{(2,1,1), (1,2,1), (1,1,2)\}$.
By Lemma \ref{lem:add_2_directions}, we have that $\|\ten T_{i,j,k}\|_2 < \kappa_2$ for $(i,j,k) \in \{(2,2,1), (2,1,2), (1,2,2)\}$.
By Lemma \ref{lem:add_3_directions}, we have that $\|\ten T_{2,2,2}\|_2 < \kappa_3$.

Combining all of these bounds, we have
\begin{align*}
f(\tup) &= R(\tup) + \sum_{i,j,k} \|\ten S_{i,j,k}(\bm A_i, \bm B_j, \bm C_k) - \ten T_{i,j,k}\|_F^2\\
&< \epsilon/2 + \kappa_0^2 + 3(K^3\sigma + d\kappa_1)^2 + 3(K^2\sigma^2 + d\kappa_2)^2 + (K\sigma^3+d\kappa_3)^2\\
&< \epsilon/2 + \epsilon/2,
\end{align*}
as desired.


\end{proof}

We now sketch our algorithm in Algorithm~\ref{alg:local}. The algorithm basically tries to follow the main Lemma~\ref{lem:robustmain}. If the point has large gradient or negative eigenvalue in Hessian, we can just use any standard local search algorithm. When the point is a higher order saddle point, we use Algorithm~\ref{alg:add_direction} as in Lemma~\ref{lem:add_2_directions} or Lemma~\ref{lem:add_3_directions} to generate directions of improvements.

\begin{algorithm}
\begin{algorithmic}
\REQUIRE tensor $\ten T$, error threshold $\epsilon$
\STATE Choose thresholds $\tau_1,\tau_2$ according to Lemma~\ref{lem:robustmain}.
\REPEAT
\STATE Run a local search algorithm to find $(\tau_1,\tau_2)$-second order stationary point.
\STATE Call Algorithm~\ref{alg:add_direction} for $i,j,k=1,2$ to generate improvement directions, repeat for $O(\log 1/\epsilon)$ times.
\IF{any of the generated directions improve the function value by at least $\somega(\sigma^{15/8})$}
\STATE Move in the direction.
\STATE Break.
\ENDIF 
\UNTIL{no direction of improvement can be found}
\end{algorithmic}
\caption{Local search algorithm for Tucker decomposition\label{alg:local}}
\end{algorithm}

Now we are ready to prove Theorem~\ref{thm:robust}

\begin{proof}[Proof of Theorem~\ref{thm:robust}]
By Lemma~\ref{lem:robustmain}, for any $(\tau_1,\tau_2)$-second order stationary point, if $f \ge \epsilon$ Lemma~\ref{lem:add_2_directions} and Lemma~\ref{lem:add_3_directions} will be able to generate a direction of improvement that improves the function value by at least $\somega(\sigma^{15/8})$ with constant probability. Since the initial point has constant loss, if a direction of improvement is found for more than $\so(1/\sigma^{15/8})$ iterations, then the function value must already be smaller than $\epsilon$.

 After the repetition, the probability that we find a direction of improvement is at least $1- o(\sigma)$. By union bound, we know that with high probability for all the iterations we can find a direction of improvement.
\end{proof}

\section{Conclusion}

In this paper we showed that the standard nonconvex objective for Tucker decomposition with appropriate regularization does not have any spurious local minima. We further gave a local search algorithm that can optimize a regularized version of the objective in polynomial time. 
There are still many open problems for the optimization of the Tucker decomposition objective.  
For example, in many applications,  the low rank tensor $\ten T$ is not known exactly. We either have significant additive noise $\ten T + \ten E$, or observe only a subset of entries of $\ten T$ (tensor completion). Local search algorithms on the nonconvex objective are able to handle similar settings for matrices~\citep{chi2019nonconvex}. We hope our techniques in this paper can be extended to give stronger guarantees for noisy Tucker decomposition and tensor completion.

\bibliographystyle{plainnat}
\bibliography{ms}

\end{document}